\documentclass[11pt]{article}

\usepackage[utf8]{inputenc}
\usepackage[T1]{fontenc}
\usepackage{amsmath,amsthm,amssymb}
\usepackage{mathtools}
\usepackage{geometry}
\usepackage{enumitem}
\usepackage{hyperref}
\usepackage{comment}
\usepackage{xcolor}
\usepackage{makecell}
\usepackage{booktabs,multirow,threeparttable,array,pifont}

\usepackage[authoryear]{natbib} % or [numbers] if you want numeric citations

\newcommand{\innerproduct}[2]{\langle #1 , #2\rangle}
\newcommand{\RN}[1]{\textup{\uppercase\expandafter{\romannumeral#1}}}

\newcommand{\Vrule}{\vrule width 0.8pt}
\usepackage{array}

\newcolumntype{P}[1]{>{\raggedright\arraybackslash}p{#1}}
\newcolumntype{M}[1]{>{\centering\arraybackslash}m{#1}}

\newcommand{\cmark}{\ding{51}}
\newcommand{\xmark}{\ding{55}}
\newcolumntype{C}[1]{>{\centering\arraybackslash}m{#1}}

\theoremstyle{definition}          % same style as Definition
\newtheorem{assumption}{}
\newtheorem{setting}{Setting}
\newtheorem{retrainingassumption}{ }% creates "Setting"

\newtheorem{theorem}{Theorem}[section]
\newtheorem{lemma}[theorem]{Lemma}
\newtheorem{corollary}[theorem]{Corollary}

\newtheorem{claim}[theorem]{Claim}
\theoremstyle{definition}
\newtheorem{definition}[theorem]{Definition}
\theoremstyle{remark}
\newtheorem{remark}[theorem]{Remark}

\newcommand{\E}{\mathbb{E}}

\newcommand{\KL}{\operatorname{D_{ KL}}}

\newcommand{\Prob}{\mathcal{P}}

\renewcommand{\Pr}{\mathbb{P}}
\newcommand{\HHP}{heterogeneous human preferences}       % lowercase phrase in running text
\newcommand{\hpnoise}{heterogeneous preference noise}    % synonym for "fluctuations"

\begin{document}

\title{Convergence and Stability Analysis of Self-Consuming Generative Models with Heterogeneous Human Curation}

% \author{%
%   Hongru Zhao\textsuperscript{1,*}\quad
%   Jinwen Fu\textsuperscript{1,**}\quad
%   Tuan Pham\textsuperscript{2,***}\\[0.5em]
%   \textsuperscript{1}School of Statistics, University of Minnesota\\
%   \textsuperscript{2}Department of Statistics and Data Sciences, UT Austin\\[0.3em]
%   \textsuperscript{*}Corresponding author \texttt{zhao1118@umn.edu} \\
%     \textsuperscript{**} \texttt{fu000217@umn.edu} \\
%     \textsuperscript{***}\texttt{tuan.pham@utexas.edu} \\
% }

\author{{Hongru Zhao}\footnote{Corresponding author} \\ \small{School of Statistics, University of Minnesota} \\  \small{\href{mailto:zhao1118@umn.edu}{zhao1118@umn.edu} } 
\and Jinwen Fu \\ \small{School of Statistics, University of Minnesota} \\  \small{\href{mailto:fu000217@umn.edu}{fu000217@umn.edu}}
\and Tuan Pham \\ \small{Department of Statistics and Data Sciences, UT Austin} \\
\small{\href{mailto:tuan.pham@utexas.edu}{tuan.pham@utexas.edu}}
}

\date{}
\maketitle
\begin{abstract}
Self-consuming generative models have received significant attention over the last few years. In this paper, we study a self-consuming generative model with heterogeneous preferences that is a generalization of the model in \citet{ferbach2024self}. The model is retrained round by round using real data and its previous-round synthetic outputs. %The generative model is iteratively retrained from a mixture of real data and synthetic data generated from previous round generative model. 
The asymptotic behavior of the retraining dynamics is investigated across four regimes using different techniques including the nonlinear Perron–Frobenius theory. Our analyses improve upon that of \citet{ferbach2024self} and provide convergence results in settings where the well-known Banach contraction mapping arguments do not apply.  Stability and non-stability results regarding the retraining dynamics are also given.

\end{abstract}

%\begin{keywords}
%curated retraining, heterogeneous preferences, Hilbert projective metric, stability
%\end{keywords}

\tableofcontents

%----------------------------------------------------------------------
\section{Introduction}

As generative models increasingly produce human-level text, images, and decisions, a central challenge is alignment: ensuring that model behavior reflects human preferences and ethical norms rather than artifacts of training data or spurious shortcuts \citep{christiano2017deep,stiennon2020learning,ouyang2022training,shen2023llmalignsurvey,skalse2022defining,perez2022redteaming,birhane2021multimodal,birhane2024laionsden,schuhmann2022laion5b}. Contemporary pipelines largely learn from preferences, often alongside scalable-oversight efforts (“superalignment” \citet{burns2023weaktostrong,kim2024road,kopf2023openassistant}), and a growing survey literature maps the practical trade-offs—from data collection and reward inference to evaluation and safety \citep[e.g.,][]{shen2023llmalignsurvey,kaufmann2023rlhfsurvey}. A common structure underlies many systems: models propose alternatives, people (or proxies) compare them, and those preferences guide the next training round \citep{shin2023benchmarks,lee2021pebble,munos2023nashlhf}.

Within this landscape, two families dominate. Reinforcement Learning from Human Feedback (RLHF) first trains a reward model from comparisons, then improves the policy via reinforcement learning with KL regularization (typically Proximal Policy Optimization (PPO)). This accommodates rich, sequence-level signals, but it introduces extra moving parts—reward modeling, on-policy sampling, and tuning—that can make training complex and sometimes unstable at scale \citep{kirk2023understanding}. Direct Preference Optimization (DPO) instead directly fits the model to human choices relative to a fixed reference policy, recasting alignment as a simpler preference-classification objective \citep{rafailov2023dpo,azar2024psipo,ethayarajh2024kto,dumoulin2023densitypref,tang2024gpo,yang2024fdpo}. DPO-style methods often match RLHF on single-turn tasks while being easier to train and reason about, precisely because the reference model provides an explicit anchor. These contrasts motivate our focus on DPO-aligned formulations and the statistical questions they raise about curate-and-retrain loops—especially how anchoring and the number of displayed candidates affect convergence and stability under heterogeneous human preferences. Repeated self-consumption and online retraining can also degrade future datasets and dynamics when feedback is narrow or synthetic \citep{hataya2023corrupt,martinez2023gaiInternet}.

\paragraph{This paper}
We cast curated retraining as population-level dynamics with two parameters: the size of the comparison pool ($K$ candidates shown per round) and the amount of reference mixing ($\alpha$, an “anchoring” weight on a fixed reference distribution).
Curators are heterogeneous and noisy, and selection follows a Plackett–Luce rule, the formal definitions appear in Section~\ref{sec:2_setup_}. 

From this setup we derive closed-form population updates and analyze convergence and stability across four regimes (defined by different $K,\alpha$ pairs). 
The reference mixing serves as a regularizer without which the process converges but is unstable to bounded reward perturbations. In the finite-pool mixed regime, when the mix weight exceeds a certain threshold, the update mapping is a Banach contraction in the total variation (TV) metric with a unique fixed point and geometric convergence; in the infinite-pool mixed regime the map locally contracts in the Hilbert projective metric and admits an explicit fixed point.  Three regimes admit closed-form limits, delineating when conservative reference mixing is required to preserve stability while maintaining preference optimization under heterogeneous raters.

Prior population analyses, most explicitly the curated self-consumption framework of \citet{ferbach2024self}, typically assume a homogeneous, noise-free curator that deterministically selects the winner. We keep the population-level viewpoint {similar to that of \cite{ferbach2024self}} but move to a more realistic setting with heterogeneous and noisy curators: selection follows a random-utility model that induces a Plackett–Luce choice rule, yielding clean closed-form population updates for both finite and infinite candidate pools \citep{mcfadden2001economic,train2009discrete}. Just as importantly, we close a practical gap left by prior theory: although real systems routinely mix curated data with a reference model (e.g., via KL anchoring or reference constraints) to improve stability \citep{stiennon2020learning,nakano2021webgpt,ouyang2022training,anthropic2022harmless,touvron2023llama}, convergence and robustness in these mixed-data settings—even under homogeneous preferences—had not been established \citep{ferbach2024self}. Our major contributions are listed as follows and in Table~\ref{tab:compare} in comparison with prior works.

%{
%Representative generative modeling paradigms and benchmarks include GANs, diffusion and score-based models, and flow matching, together with large text-to-image, video, and audio systems, and open image–text corpora \citep{goodfellow2014gan,ho2020ddpm,song2020score,lipman2022flowmatching,shaul2023kinetic,villani2009ot,ramesh2021dalle,villegas2022phenaki,borsos2023audiolm,schuhmann2022laion5b,scholar_journeydb2023.
%{ does this paragraph deliver any point?}
%}

\subsection*{Main Contributions}
\begin{itemize}[leftmargin=*, itemsep=3pt]
\item \textbf{Heterogeneous rater model.}% {We propose a} random-utility curation model with between-rater heterogeneity and within-rater noise, combined with conservative mixing against a reference source. { This is a generalization of the model considered in \cite{ferbach2024self}}.
{We propose a random-utility curation model with between-rater heterogeneity and within-rater noise, combined with conservative mixing against a reference source; see \eqref{equ:retraining_mixture} and \eqref{equ:retraining_update} for more details.  This is a generalization of the model considered in \cite{ferbach2024self}.}

\item \textbf{Closed--form dynamics in four regimes.}% { We derive } exact population updates for finite vs.\ effectively unbounded candidate pools { why not just say finite vs infinite candidate pools}, with and without mixing. Minimal sufficient conditions are established for the existence of such updates. 
{We derive  the exact population updates for finite vs. infinite candidate pools, and with or without mixing. These four regimes and the exact updates are described in detail in Sections \ref{alpha=0} and \ref{alpha>0}. Minimal sufficient conditions are established for the existence of such updates. The details can be found in Section \ref{sec:3_population}.}

\item \textbf{Convergence and stability.} 
We establish monotonic improvement in the expected exponential reward for self-consuming generative models. {In the purely synthetic retraining, we establish the convergence in Kullback–Leibler (KL) divergence. In the mixed retraining with a finite candidate pool, we show a contraction result regarding the total variation distance between the update and the fixed point, provided that the  reference‑mix proportion is large enough. In the mixed retraining with an infinite candidate pool, we prove the convergence in the Hilbert projective metric to a fixed point. Finally, we provide stability and instability results of the retraining dynamics.}
%\item \textbf{Bridging prior gaps.} Guarantees for mixed--data curated retraining under heterogeneous preferences, extending homogeneous, deterministic analyses and aligning theory with practice.
\end{itemize}
{The rest of the paper is organized as follows. The model setup, retraining procedures and motivations are provided in Section \ref{sec:2_setup_}. We collect some technical properties and formulas of the updates in Section \ref{sec:3_population}. Convergence results can be found in Section \ref{sec:convergence analysis} and Section \ref{subsec:unstable} contains the stability and instability results of the retraining dynamics. }

\begin{table}[htbp]%\label{tab:compare}
\centering
\small
\setlength{\tabcolsep}{6pt}
\renewcommand{\arraystretch}{1.22}
\begin{threeparttable}
\caption{Comparison with prior results.}
\label{tab:compare}
\begin{tabular}{C{0.25\linewidth} C{0.35\linewidth} C{0.24\linewidth} C{0.08\linewidth}}
\toprule
\multirow{2}{*}{Aspect} & \multirow{2}{*}{Case} & \multicolumn{2}{c}{Covered?} \\
\cmidrule(lr){3-4}
 & & \textsc{\cite{ferbach2024self}} & \textsc{Ours} \\
\midrule
Setting & raters with heterogeneous preferences & \xmark & \cmark \\
\addlinespace[2pt]
\multirow{2}{*}{ Improvement}
  & $\alpha=0$ (pure self-consuming) & \cmark & \cmark \\
  & $\alpha>0$ (mixing with reference) & \xmark & \cmark \\
\addlinespace[2pt]
\multirow{2}{*}{Convergence analysis}
  & $\alpha=0$ (pure self-consuming) & \cmark & \cmark \\
  & $\alpha>0$ (mixing with reference) & \xmark & \cmark \\
\addlinespace[2pt]
\multirow{2}{*}{Stability / robustness}
  & $\alpha=0$ & \xmark & \cmark \\
  & $\alpha>0$ & \xmark & \cmark \\
\bottomrule
\end{tabular}
\begin{tablenotes}[flushleft]
\footnotesize
\item \emph{Notation:} $\alpha$ is the mixing weight (regularization proportion) on the reference distribution, corresponds to $1/(1+\lambda)$ in \cite{ferbach2024self}: $\alpha=0$ means no mixing, $\alpha>0$ means curated data are mixed with a reference. The improvement is measured on the expected exponential reward, and the convergence analysis is conducted on the iteratively curated distribution. \cmark\ denotes the paper provides explicit analysis; \xmark\ indicates it does not.
\end{tablenotes}
\end{threeparttable}
\end{table}

\section{Model Setup}
\label{sec:2_setup_}

% Motivation paragraph explaining the problem context and contributions
In RLHF and DPO methods, generative models are iteratively updated based on human‑curated samples. Previous work shows that such curated models can drift or become unstable, and existing theoretical analyses do not fully address convergence when feedback is heterogeneous or when curated data are mixed with a reference distribution.  Our goal is to close this gap by analyzing a discrete‑$K$ curation loop under heterogeneous human preferences, deriving convergence properties and highlighting the role of a reference distribution as a regularizer.

%----------------------------------------------------------------------
\noindent \textbf{Discrete-$K$ choice curation model under \HHP.}

Let $(\mathcal{X},\mathcal{B},\pi)$ be a measure space, where $\pi$ is $\sigma$-finite. Let \(X\in\mathcal X\) be a random element with density \(p\) with respect to (w.r.t.) $\pi$. In each update, we draw $K$ i.i.d. samples $X_{1:K}=(X_1,\ldots,X_K)$ from $p$ (we write $x_{1:K}$ for a realization), and submit them to a heterogeneous human curation step that returns the index $I\in [K]:=\{1,2,\ldots,K\}$ of the preferred (model-generated) candidate; we denote the curated sample by $\widehat{x}=x_I$. To model heterogeneous human selection, we adopt Luce’s choice rule \citep{luce1963handbook}, positing a measurable reward \(r:\mathcal X \to\mathbb{R}\) and \hpnoise\ (individual-preference fluctuations) \( \bigl\{\varepsilon_i(\cdot )\bigr\}_{i=1}^K,\) where $\varepsilon_1,\ldots,\varepsilon_K$ are independent random fields, independent of $x_{1:K}$. Depending on the application, the reward function $r(x)$ can represent either human values and preferences (for alignment purposes) or domain-specific expert judgment (for performance improvement in specialized areas). The utility of candidate $x_k$ is $\tilde{w}_k:= e^{r(x_k)+\varepsilon_k(x_k)}$. Under the Plackett–Luce specification \citep{luce1959individual,plackett1975analysis}, the curation procedure draws an index $I\in [K]$ with 
\begin{equation}\label{equ:1_LP_index_selection}
  \mathbb{P}(I=k|x_{1:K},\varepsilon_{1:K})=\frac{\tilde{w}_k}{\sum_{i=1}^K\tilde{w}_i},  
\end{equation}
and sets $\hat{x}=x_I$.

The \hpnoise\  captures both within‑rater stochasticity and between‑rater taste differences
(i.e., preference heterogeneity in the random‑utility sense \citep{mcfadden2001economic,train2009discrete}). We update the generative model either solely from the curated samples or by mixing the curated-data distribution with a reference distribution of density $p_{\rm ref}$ (w.r.t. $\pi$), using mixture weight $\alpha \in [0,1)$. Preference heterogeneity is widely observed and shaped by population composition \citep{henrich2010weird}, while classical generalization in psychological spaces motivates distance‑sensitive choice behavior \citep{shepard1957generalization}.

Our goal is then to study the convergence property of the self-consuming loop \citep{ferbach2024self}, that is, in each round the model generates candidates, the curation procedure selects one, and training uses curated plus optional reference data.

\paragraph{t-th round of curation-update}
\begin{enumerate}[label=\textbf{Step \arabic*:}, leftmargin=*]
  \item Draw $X_1,\dots,X_K\stackrel{i.i.d.}{\sim}p_t$.
  \item Evaluate the utility at each sample $r(X_k)+\varepsilon_k(X_k)$ and set $\tilde{w}_k=e^{r(X_k)+\varepsilon_k(X_k)},$ $1\leq k\leq K$.
  \item Pick $\widehat X:=X_I$ with probability \eqref{equ:1_LP_index_selection}, and denote the distribution by $\mathcal{PL}(X_{1:K},\varepsilon_{1:K} )$.
  \item Train the new generator $p_{t+1}$ from a mixture of the law from the curated sample $\hat X$ and the law $p_{\rm ref}$ of $X_{\rm{ref}}$:
  \begin{gather}
      \mathcal L(p;p_t,K,\alpha)
  :=\alpha\E_{X_{\rm ref}\sim p_{\rm ref} }\bigl[\log p(X_{\rm ref} )\bigr]
    +(1-\alpha)\E_{\substack{
      X_{1},\dots,X_{K}\sim p_t, \\
      \widehat{X}\sim \mathcal{PL}(X_{1:K},\varepsilon_{1:K} )
  }} \bigl[\log p(\widehat X)\bigr],\label{equ:retraining_mixture}\\
  p_{t+1}=\arg\max_{p \in \mathcal{P}_{\pi} }\mathcal L(p;p_t,K,\alpha)\label{equ:retraining_update},
  \end{gather}
\end{enumerate}
where $\mathcal{P}_{\pi}$ is the set of distributions in our model, which has been defined in Section~\ref{sec:3_population}. The procedure starts with an initial density $p_0$ and is iteratively regularized by a reference density $p_{\rm ref}$. 
In applications, we can take $p_0$ to be a pre-trained generative model. 
Recall that $\alpha\in[0,1)$ represents the mixing weight (regularization proportion), with $\alpha=0$ representing no mixing/regularization. At $\alpha=1$, the model collapses to $p_{\rm ref}$ each round. 
The reference density $p_{\rm ref}$ serves as a regularizer: it may be chosen as a high-quality estimate (e.g., a smoothed estimate of the real data), or simply the initial density itself—i.e., $p_{\rm ref}=p_0$. 

Choosing $p_{\rm ref}=p_0$ implements conservative regularization by mixing the curated distribution with the initial density at each round.
Here $K$ denotes the number of model‑generated candidates displayed to the human curator in each round (the comparison‑pool size).
The update of $p_{t+1}$ and the convergence of the sequence $\{p_t\}_{t=1}^\infty$ vary in different settings of $(\alpha, K)$. Specifically, we discuss \textbf{four Regimes}: 
\begin{itemize}
    \item (i) $\alpha=0$, $K<\infty$;
    \item (ii) $\alpha=0$, $K=\infty$;
    \item (iii) $\alpha\in(0,1)$, $K<\infty$; 
    \item (iv) $\alpha\in(0,1)$, $K=\infty$.
\end{itemize}  
For $K=\infty$, interpret $\{(X_k,\varepsilon_k)\}_{k=1}^\infty$ as an infinite sequence. Although the setting \(K = \infty\) is not attainable in practice—no human can curate infinitely many candidates—it plays a fundamental theoretical role to describe the situation with optimal curation. 

Table~\ref{tab:regimes} {describes} the four population-level retraining regimes we study {in this paper}, listing for each (a) the population update, (b) the assumptions invoked, (c) the convergence metric, (d) whether a closed-form limit is available, and (e) stability to bounded reward perturbations.
%The curation kernels $H^{K}_{p}$ and $H^{\infty}_{p}$ referenced in the table are defined in \eqref{eq:H_K} and \eqref{equ:def_H_infty}, respectively.

\begin{table}[htbp]
\centering
\caption{Theoretical results in four regimes.}
\small
\setlength{\tabcolsep}{3pt}
\renewcommand{\arraystretch}{1.18}
\begin{tabular}{M{0.11\linewidth}!{\Vrule} M{0.16\linewidth}!{\Vrule} M{0.32\linewidth} M{0.38\linewidth}}
\toprule
 &  & \multicolumn{2}{c}{\textbf{Retraining mixture}} \\
\cmidrule(lr){3-4}
\centering\textbf{Regime} & \centering\textbf{Entry} & $\boldsymbol{\alpha=0}$ & $\boldsymbol{\alpha\in(0,1)}$ \\
\midrule

\multirow{5}{*}{$K<\infty$}
& Dynamics
  & $p_{t+1}(x)=p_t(x)\,H^{\!K}_{p_t}(x)$
  & $p_{t+1}(x)=\alpha\,p_{\rm ref}(x)+(1-\alpha)\,p_t(x)\,H^{\!K}_{p_t}(x)$ \\

& Assumptions
  & \ref{re_ass:1_Shannon-entropy}–\ref{re_ass:3_ess-bdd-Q}; \ref{setting:variable-noise}; \ref{ass1:non0-mass}.
  & \ref{re_ass:1_Shannon-entropy}; \ref{setting:variable-noise}. \\

& Metric
  & KL divergence
  & TV \;(for $\alpha>\frac{K-1}{K}$). \\

& Explicit limit
  & Yes
  & No (but the limit is the unique fixed point in \eqref{equ:non_linear_oper}). \\

& Stability
  & Unstable TV to bounded reward perturbations
  & Stable in TV \;(for $\alpha>\frac{K-1}{K}$). \\
\midrule

\multirow{5}{*}{$K=\infty$}
& Dynamics
  & $p_{t+1}(x)=p_t(x)\,H^{\!\infty}_{p_t}(x)$
  & $p_{t+1}(x)=\alpha\,p_{\rm ref}(x)+(1-\alpha)\,p_t(x)\,H^{\!\infty}_{p_t}(x)$ \\

& Assumptions
  & \ref{re_ass:1_Shannon-entropy}–\ref{re_ass:3_ess-bdd-Q}; \ref{setting:process-noise}; \ref{ass1:non0-mass}.
  & \ref{re_ass:1_Shannon-entropy}–\ref{re_ass:3_ess-bdd-Q}; \ref{ass:A2_new}, \ref{ass:A3}; \ref{setting:process-noise}. \\

& Metric
  & KL divergence
  & Hilbert projective metric \\

& Explicit limit
  & Yes
  & Yes (up to a scalar normalizer $c_\ast$). \\

& Stability
  & Unstable TV to bounded reward perturbations.
  & Stable in TV under \ref{ass:4_perturbation}. \\
\bottomrule
\end{tabular}
\begin{tablenotes}[flushleft]
\footnotesize
\item \emph{Metric hierarchy.} On the interior of the positive cone, the Hilbert–projective metric is stronger than {the} KL {divergence}, which is stronger than {the} TV {distance} \citep{atar1997exponential}.
\end{tablenotes}

\label{tab:regimes}
\end{table}

\paragraph{Why we call it regularization.}
Mixing the curated distribution with a fixed reference $p_{\rm ref}$ plays the role of a regularizer: it shrinks the update toward an anchor and turns the population map into a contraction. In the finite‑$K$ case, when the mix proportion satisfies $\alpha>1-\tfrac{1}{K}$, the update operator is a Banach contraction in the TV metric, yielding a unique fixed point and geometric convergence (Theorem~\ref{thm:tv-convergence-iii}). In addition, the entire trajectory is Lipschitz‑stable in the TV metric to bounded reward perturbations (Theorem~\ref{thm:stable_1}). In the $K=\infty$ regime, the regularized update becomes a positive linear map with local projective contraction and continuous dependence on bounded reward perturbations (Theorem~\ref{thm:Regime_iv_convergence}; Theorem~\ref{thm:stable_2}). Intuitively, the $p_{\rm ref}$ term damps selection‑amplified fluctuations, which is exactly the stabilizing effect we expect from a regularizer.

\begin{remark}[Heterogeneity and relation to DPO]
Empirically, preference datasets exhibit substantial rater disagreement and between‑user variation, both in summarization and assistant settings \citep{stiennon2020learning,anthropic2022harmless,ouyang2022training}; recent work even learns from heterogeneous feedback via personalization and aggregation \citep{park2024rlhfhetero}; see also \citet{henrich2010weird} for broader evidence of heterogeneous populations. Our analysis keeps this variability explicit and shows that the convergence/stability guarantees survive under \HHP{}. Moreover, when $K=2$ our curated self‑consuming update uses the same Bradley–Terry \citep{bradley1952rank} preference model that underlies DPO \citep{rafailov2023dpo}. With conservative mixing against a reference distribution, our update plays a role analogous to KL anchoring explored in generalized preference‑optimization frameworks \citep{tang2024gpo,yang2024fdpo}. We view our method as a population‑level, DPO‑like dynamic rather than a strict instantiation of the DPO loss \citep{hong2024orpo,meng2024simpo,pang2024irpo,qi2024ofsdpo,shi2024dmpo}.
\end{remark}

We consider two models of human-preference noise that reflect real–world variability. Let $\{\varepsilon_k(\cdot)\}_{k=1}^K$ be i.i.d. copies of a real-valued random field $\varepsilon(\cdot)$ on $\mathcal{X}$. 
We distinguish whether the noise may vary with the outcome \(x\) or has \(x\)-invariant one-point marginals.

\begin{setting}[Nonstationary Preference Noise]\label{setting:process-noise}
%The noise is an arbitrary real-valued random field \(\{\varepsilon(x): x \in \mathcal{X}\}\).
The \hpnoise\ is an arbitrary real‑valued random field \(\{\varepsilon(x):x\in\mathcal X\}\).
\end{setting}

\begin{setting}[Stationary Preference Noise]\label{setting:variable-noise}
There exists a random variable \(\varepsilon\) with \(x\)-invariant one‑point marginals:
\(\varepsilon(x)\overset{d}{=}\varepsilon\) for all \(x\in\mathcal X\).

% The noise has \(x\)-invariant one-point marginals: for some random variable \(\varepsilon\),
% \[
% \varepsilon(x)\overset{d}{=}\varepsilon \quad \text{for all } x\in\mathcal{X}.
% \]
% Equivalently, the marginal distribution of \(\varepsilon(x)\) does not depend on \(x\).
%This condition does not impose independence across \(x\) nor full (strict) stationarity of the field.
\end{setting}
Notice that Setting~\ref{setting:variable-noise} is a special case of Setting~\ref{setting:process-noise}. 
Unless stated otherwise, we work under Setting~\ref{setting:variable-noise}.

\subsection{Notations}\label{subsec:notation}
We work on a measurable space $(\mathcal{X},\mathcal{B},\pi)$ with $\sigma$-finite baseline measure $\pi$. We write $\mathcal{P}_{\pi}$ as the set of probability measures that are absolutely continuous w.r.t. $\pi$ and the measures are identified with their a.e.-unique densities. Expectations $\E[\cdot]$ are taken under the law indicated in the subscript. For probability measures $\mu,\nu$, absolute continuity is denoted $\mu\ll\nu$, and essential suprema/infima are taken w.r.t.\ the ambient measure made explicit in context.

With the iterative model update, the model density at iteration $t$ is $p_t$, with initialization $p_0$ and a fixed reference density $p_{\rm ref}$. $\alpha\in[0,1)$ denotes the retraining mixture weight (regularization proportion). The (heterogeneity-averaged) utility is $Q(x):=\,\E[\exp\{r(x)+\varepsilon(X)\}\mid X=x]$, where $r:\mathcal X\to\mathbb R$ is the reward and $\varepsilon$ captures individual noise. Let \(\mathbb P_0\) denote the probability measure on \((\mathcal X,\mathcal B)\) induced by the density \(p_0\) with respect to \(\pi\).

We set $Q_*:=\mathbb{P}_0\operatorname*{-ess\,sup}_x Q(x)$ and $Q_{\min}:=\operatorname*{ess\,inf}_x Q(x)$, and use $A:=\{x\in\mathcal X:\ Q(x)=Q_*\}$ to denote the maximizer set of the utility and $\mathbf{1}_A(\cdot)$ be the indicator function on the set $A$. We denote the number of individual raters as $K$ and write $[K]=1,\cdots,K$. %For discrete-$K$ curation, we define the choice kernel (to be explained later) as
%\[
%H_p^{K}(x)
%=\E\!\left[K\,
%\frac{e^{r(x)+\varepsilon(x)}}{e^{r(x)+\varepsilon(x)}+\sum_{k=1}^{K-1} e^{r(X_k)+\varepsilon_k(X_k)}}\right],
%\]
%and its infinite-pool limit (when it exists) is $H_p^\infty(x)=Q(x)/\E_{X\sim p}[Q(X)]$. We denote the curated mass by $S_p(x):=p(x)\,H_p^{K}(x)$ (that is proved to be a density) and the associated measure by $\mathbb S_p(B):=\int_B S_p(x)\,\pi(dx)$. 

For part of the analysis, {unless stated otherwise}, we use $\mathbb P_{\rm ref}$ to denote the probability measure with density $p_{\rm ref}$ so that $d\mathbb P_{\rm ref}=p_{\rm ref}\,d\pi$, and define $\mathcal P_{\mathbb P_{\rm ref}}:=\{\mu:\mu\ll\mathbb P_{\rm ref}\}$. Define the positive cone $\mathcal K:=\{f\in L^1(\mathbb P_{\rm ref}): f>0\ \mathbb P_{\rm ref}\text{-a.s.}\}$. When $p_t\in\mathcal{P}_{{\mathbb P}_{\rm{ref}}}$, we reweight it by the reference via $w_t:=p_t/p_{\rm ref}$ and use the inner product $\langle f,g\rangle_{\rm ref}:=\int f(x)g(x)\,d\mathbb P_{\rm ref}(x)$. 

We measure discrepancies between densities using the TV metric $d_{\rm TV}(p,q):=\tfrac12\int|p-q|\,d\pi$ and KL divergence $\KL(p\|q):=\int p\log(p/q)\,d\pi$ (when $p\ll q$). 
% And in special cases with positive functions $u,v\in\mathcal K$, we use the Hilbert projective metric, which is
% \[
% \beta(u,v):=\operatorname*{ess\,sup}_x \frac{u(x)}{v(x)},\quad
% \alpha(u,v):=\operatorname*{ess\,inf}_x \frac{u(x)}{v(x)}=\beta(v,u)^{-1},\quad
% d_{\mathcal H}(u,v):=\log\!\frac{\beta(u,v)}{\alpha(u,v)}.
% \]

\section{Population-Level Retraining Dynamics with Curated Synthetic Data}\label{sec:3_population}
This section derives the population update maps curation under heterogeneous preferences—first without mixing/regularization and then with conservative mixing—covering finite and mean field (infinite-pool) regimes.

Define
\[
\mathcal{P}_{\pi}
\;:=\;
\bigl\{ \mu \text{ is a probability measure on } (\mathcal{X},\mathcal{B}) : \mu \ll \pi\,\bigr\},
\]
where $\mu \ll \pi$ means $\mu $ is absolutely continuous with respect to the baseline $\sigma$-finite $\pi$. By the Radon--Nikodym theorem, every $\mu \in \mathcal{P}_{\pi}$ admits a density $p := \frac{d\mu}{d\pi}$, unique $\pi$-a.e.\ \citep[see, e.g.,][]{van2000asymptotic}.
With a slight abuse of notation, we will identify $\mu$ with its density and write $p\in\mathcal{P}_{\pi}$ when convenient.

%We next state mild integrability conditions that ensure the retraining objective \eqref{equ:retraining_update} is well-posed and its maximizer admits a closed-form update.

We impose the following Retraining Assumptions to ensure the existence of an updated distribution $p_{t+1}$ from maximizing the retraining objective \eqref{equ:retraining_update}.

\begin{retrainingassumption}[Finite Shannon entropy]\label{re_ass:1_Shannon-entropy}
    The densities \(p_{0},p_{\rm{ref}}\in\mathcal{P}_\pi\) have finite Shannon entropy, that is,
\[
 -\infty<h(p) := -\int p(x)\log p(x)\,\pi(dx)<\infty\text{ for } p\in\{p_0,p_{\rm{ref}}\}.
\] It equivalently demands \(\int p(x)|\log p(x)|\,\pi(dx)<\infty\).
\end{retrainingassumption}

\begin{retrainingassumption}[Exponential integrability]\label{re_ass:2_exp-intgr}
For every $p\in\mathcal{P}_\pi$,
\(
  \E_{X\sim p} \bigl[e^{\,|r(X)+\varepsilon(X)|}\bigr] < \infty .
\)
\end{retrainingassumption}

\begin{retrainingassumption}[Finite expected utility]\label{re_ass:3_ess-bdd-Q}
Define the heterogeneity‑averaged utility for \(x\in\mathcal{X}\) as \(Q(x):= e^{r(x)}\mathbb{E}\left[e^{\varepsilon(X)}\mid X=x\right].\)
Let $Q_*$ denote the essential upper bound for $Q(x)$, that is, 
\[ Q_{*}:= \mathbb P_0 \operatorname*{-ess\,sup}_{x\in\mathcal X} Q(x) = \inf\bigl\{M\in\mathbb{R}:\mathbb{P}_0\bigl(\{x:\, Q(x)>M\}\bigr)=0\bigr\}.\]
Assume \(0<Q_{*}<\infty\). Without loss of generality, we may assume $Q(x)\leq Q_*$ for all $x\in \mathcal{X}$ (by choosing a version of $x\mapsto \left[e^{\varepsilon(X)}\mid X=x\right]$ and modifying it on a $\pi$-null set).
\end{retrainingassumption}
Here $Q$ is the heterogeneity‑averaged exponential reward; the conditional expectation averages over \HHP{}, yielding a population‑level choice kernel.

\subsection{Retraining Dynamics Without Regularization} \label{alpha=0}
In this section, we analyze the dynamics of the discrete-$K$ choice model of the curation procedure in the fully synthetic setting ($\alpha=0$). We start with a finite number of displayed candidates $K<\infty$ (Regime (i)) and then extend the updating formula to the effectively infinite‑candidate case $K=\infty$ (Regime (ii)).

When $K<\infty$, the distribution of curated data in the discrete-$K$ choice model is characterized by the choice kernel, given $p\in \mathcal{P}_{\pi}$,
\begin{equation}\label{eq:H_K}
    H_p^{K}(x):=
      \E_{\substack{X_{1},\dots,X_{K-1}\sim p\\
       \varepsilon,\varepsilon_{1 },\ldots,\varepsilon_{K-1 } }} \Bigl[
        K\,\frac{e^{\,r(x)+\varepsilon(x)}}{%
               e^{\,r(x)+\varepsilon(x)}+
               \sum_{k=1}^{K-1}e^{\,r(X_k)+\varepsilon_k(X_k)}}
      \Bigr].
  \end{equation}

We first establish the following lemma to characterize the distribution of the curated sample $\widehat{X}$. {Recall the Plackett–Luce specification $\mathcal{PL}(X_{1:K},\varepsilon_{1:K})$ defined as in \eqref{equ:1_LP_index_selection}. This leads to the following result.}

\begin{lemma}\label{lem:curated-distribution}
%    Under either Setting~\ref{setting:process-noise} or Setting~\ref{setting:variable-noise}, 
 Let $\widehat X\sim \mathcal{PL}(X_{1:K},\varepsilon_{1:K})$ be the curated sample where $X_{1},\ldots,X_K\overset{\text{i.i.d.}}{\sim}p$. For any measurable function $f$ with $\E_{X\sim p}[|f(X)|]<\infty$,
\begin{equation}\label{equ:5_functional_int}
    \mathbb{E}[f(\widehat X)]=\E_{X\sim p}[f(X)H_p^K(X)].
\end{equation}    
As a consequence, the density (w.r.t. $\pi$) of $\widehat X$ is $p H_p^K$.

%$pH_p^K$ is the density function (w.r.t. $\pi$) of $\widehat X$.

\end{lemma}

Lemma \ref{lem:curated-distribution} describes the distribution of the curated data through the choice kernel. Then with the following theorem, the updating formula is explicitly derived. Recall that Regime (i) corresponds to the setting $\alpha=0,K<\infty$.

\begin{theorem}[Update Rule under Regime (i)]%
\label{thm:pure-Kfinite}
Under Regime (i) and Retraining Assumption \ref{re_ass:1_Shannon-entropy}, the updating formula is given by
\begin{equation}\label{equ:update_regime_1}
    p_{t+1}(x)\;=\;p_{t}(x)\,H_{p_t}^{K}(x),
  \qquad x\in\mathcal X,
\end{equation}
as the unique maximizer of the pure self‑consuming objective \eqref{equ:retraining_mixture}
      \begin{align}\label{eq:finite-K-objective}
          p\longmapsto
          \mathcal L(p;p_t,K,0) =\mathbb E_{\substack{X_{1},\dots,X_{K}\sim p_t\\
      \widehat{X}\sim \mathcal{PL}(X_{1:K}, \varepsilon_{1:K})}} \bigl[\log p(\widehat X)\bigr] =\E_{X\sim p_t} \bigl[\log p(X) H^K_{p_t}(X)\bigr]
       \end{align} for any $t\in\mathbb N$,
where the choice kernel \(H_{p_t}^{K}\) is given by~\eqref{eq:H_K}, and $\widehat X$ is the curated sample from the discrete-$K$ choice model under density $p_t$.
\end{theorem}

In the infinite pool ($K=\infty$) situation, the update objective is not induced by a finite‑sample choice; we therefore define its $K\to \infty$ limit inside the expectation of \eqref{eq:finite-K-objective}. In the $t-$th step, define
\begin{align}\label{eq:inf-K-objective}
    L(p;p_t,\infty,0):=\E_{X\sim p_t}   \bigl[\log p(X) H^\infty_{p_t}(X)\bigr],
\end{align}
where
\begin{equation}\label{eq:Hinf}
    H_p^{\infty}(x):=\lim_{K\to\infty} H_p^{K}(x)
\end{equation} is given by the corresponding choice kernel. The following lemma studies the existence and analytical expression of the kernel $H_p^{\infty}$. Taking \(K \to \infty\) removes finite-sample randomness in the choice kernel and yields a closed-form expression. This limit serves as an asymptotic surrogate of large-\(K\) behavior, analogous to the infinite-width limit in neural-network theory (e.g., mean-field or NTK analyses), which provides tractable mean-field dynamics and reveals underlying convergence and stability structure. 

In the same spirit, analyzing the \(K = \infty\) regime enables us to isolate the deterministic, population-level effects of the curation mechanism and to derive explicit fixed point formulas and convergence guarantees that would be intractable for finite \(K\).

\begin{lemma}\label{lem:Hlimit}
Under Retraining Assumption \ref{re_ass:2_exp-intgr}, the choice kernel is
\begin{equation}\label{equ:def_H_infty}
    H_{p}^{\infty}(x)=\lim_{K\to\infty}\E_{\substack{X_{1},\dots,X_{K-1}\sim p\\
       \varepsilon,\varepsilon_{1 },\ldots,\varepsilon_{K-1 } }}
       \Bigl[
        K\,\frac{e^{\,r(x)+\varepsilon(x)}}{%
               e^{\,r(x)+\varepsilon(x)}+
               \sum_{k=1}^{K-1}e^{\,r(X_k)+\varepsilon_k(X_k)}}
      \Bigr]=\frac{Q(x)}{\mathbb{E}_{X\sim {p}}Q(X)}.
\end{equation}

\end{lemma}

%Then the update with infinite pool:
%\begin{equation} 
%p_{t+1}=\arg\max_{p}L(p;p_t,\infty,\infty)=\arg\max_{p}\E_{X\sim p_t}   \bigl[\log p(X) H^\infty_{p_t}(X)\bigr].
%\end{equation}

Recall that Regime (ii) corresponds to the setting $\alpha=0,K=\infty$. When the retraining objective is well-defined, the following theorem gives the updating formula with respect to the choice kernel.

\begin{theorem}[Update rule under Regime (ii)]%
\label{thm:pure-Kinfty}
Under Regime (ii) and Retraining Assumption  \ref{re_ass:1_Shannon-entropy}, \ref{re_ass:2_exp-intgr} and \ref{re_ass:3_ess-bdd-Q}, the update is given by \[
  p_{t+1}(x)\;:=\;p_{t}(x)\,H_{p_t}^{\infty}(x),
  \qquad x\in\mathcal X,
\]
as the unique maximizer (over $p$) of  the functional on the right-hand side of \eqref{eq:inf-K-objective}, for any $t\in\mathbb N$.
\end{theorem}

In other words, Theorem \ref{thm:pure-Kinfty} states that 
\[
p_{t+1}=\mbox{argmax}_{p} \E_{X\sim p_t}   \bigl[\log p(X) H^\infty_{p_t}(X)\bigr].
\]

\subsection{Retraining Dynamics With Regularization} \label{alpha>0}

When $\alpha\in (0,1)$, the regularization is applied to every update, according to \eqref{equ:retraining_mixture}, and the retraining objective characterized by the choice kernel in the $t-$th step becomes
\begin{align}\label{eq:objective-mix}
    \mathcal{L}(p;p_t,K,\alpha)=\alpha\E_{X_{\rm{ref}}\sim p_{\rm{ref}}}[\log p(X_{\rm{ref}})]+(1-\alpha)\E_{p_t}[\log p(X)H_{p_t}^K(X)],~~K\in\mathbb N\cup\{+\infty\}.
\end{align}

Recall that Regime (iii) corresponds to the setting $\alpha\in(0,1),K<\infty$ and Regime (iv) corresponds to the setting $\alpha\in(0,1),K=\infty$.

\begin{theorem}[Update rule under Regime (iii)/(iv)]\label{thm:update-mix}
If Regime (iii) holds together with Retraining Assumption~\ref{re_ass:1_Shannon-entropy}, 
or if Regime (iv) holds together with Retraining Assumptions~\ref{re_ass:1_Shannon-entropy}-\ref{re_ass:3_ess-bdd-Q}, then the update is given by 
%Under either Regime (iii) and Retraining Assumption \ref{re_ass:1_Shannon-entropy}, or Regime (iv) and Retraining Assumptions \ref{re_ass:1_Shannon-entropy}, \ref{re_ass:2_exp-intgr} and \ref{re_ass:3_ess-bdd-Q}, the update is given by
\begin{equation}\label{equ:update_Regime_iii_iv}
      {\;p_{t+1}(x)   =\alpha p_{\mathrm{ref}}(x)   +(1-\alpha) p_t(x)\,H_{p_t}^K(x)   \;}
\end{equation}
as the unique maximizer (over $p$) of the functional on the right-hand side of \eqref{eq:objective-mix}, for any $t\in\mathbb N$.
\end{theorem}
Similar to Theorem \ref{thm:pure-Kinfty}, the result in Theorem \ref{thm:update-mix} states that 
\[
p_{t+1}=\mbox{argmax}_{p} \left\{ \alpha\E_{X_{\rm{ref}}\sim p_{\rm{ref}}}[\log p(X_{\rm{ref}})]+(1-\alpha)\E_{X\sim p_t}[\log p(X)H_{p_t}^K(X)] \right\}.
\]

\section{Convergence Analysis} \label{sec:convergence analysis}
With the update dynamics $\{p_t(x)\}$ worked out in the previous Sections, we now study convergence as $t\to \infty$.

\subsection{Convergence in pure synthetic data retraining}
In addition to the global assumptions in the former sections, we impose the following assumption on the reward function and initial density, which is closely related to the limiting distribution.

\begin{assumption}[Positive initial probability of the maximum reward]\label{ass1:non0-mass}
Let the reward maximizing set be \(A := \{x\in\mathcal{X} : Q(x)=Q_{*}\}\).
Define the probability measure \(\mathbb{P}_0\) with density \(p_0\) w.r.t.~\(\pi\) by
\(\mathbb{P}_0(B) := \int_B p_0(x)\,\pi(dx)\), $B\in \mathcal{B}$.
Assume \(\mathbb{P}_0(A) > 0\).
\end{assumption}

With this assumption, we can state the limiting density for the pure synthetic data retraining case ($\alpha=0$). Define
\begin{equation}\label{equ:limit_dist_pure_syn}
    p_*(x) = \frac{p_0(x)\mathbf{1}_{A}(x)}{p_0(A)}.
\end{equation}
When we ignore the individual‑preference fluctuations--i.e., $\varepsilon(x)\equiv0$, the limiting distribution $p_*$ offered in \eqref{equ:limit_dist_pure_syn} coincides with the limiting distribution obtained for iterative retraining in \cite{ferbach2024self}.  

Our convergence analysis regarding Regime (i) assumes that the noise is stationary; see Setting \ref{setting:variable-noise} above. Note that if we let the noise to be zero, the results below in Regime (i) (Lemma \ref{lem:Ri_Rt-to-Qstar})  recover that of \cite{ferbach2024self}.  In the Regime (ii) (Lemma \ref{lem:Rii_Rt-to-Qstar}), we are able to obtain convergence result even for nonstationary rewards; see Setting \ref{setting:process-noise} for more details.

% Due to the technical difficulty in Regime (i), we can only consider the stationary reward fluctuation Setting \ref{setting:variable-noise}. Even so, Setting \ref{setting:variable-noise} is generalized enough to recover the convergence result in \cite{ferbach2024self} by setting individual-preference fluctuations ($\varepsilon(x)\equiv0$).

% 

\begin{lemma}\label{lem:Ri_Rt-to-Qstar}
Under Regime (i), Setting \ref{setting:variable-noise}, Retraining Assumption \ref{re_ass:1_Shannon-entropy}–\ref{re_ass:3_ess-bdd-Q}, and Assumption \ref{ass1:non0-mass}, we have expected exponential reward increases at each retraining iteration
\[
  \E_{X\sim p_{t+1}}Q(X)\geq \E_{X\sim p_{t}}Q(X).
\]
Moreover, 
\[
  \qquad \lim_{t\to\infty} \E_{X\sim p_t} Q(X)= Q_*.
\]
\end{lemma}
\begin{lemma}\label{lem:Rii_Rt-to-Qstar}
Under Regime (ii), Setting \ref{setting:process-noise}, Retraining Assumption \ref{re_ass:1_Shannon-entropy}–\ref{re_ass:3_ess-bdd-Q}, and Assumption \ref{ass1:non0-mass}, we have expected exponential reward increases at each retraining iteration
\begin{equation}
    \mathbb{E}_{X\sim p_{t+1}}[Q(X)] \;\geq\; \mathbb{E}_{X\sim p_{t}}[Q(X)]+\frac{\operatorname{Var}_{X\sim p_{t}}(Q(X))}{\mathbb{E}_{X\sim p_t}[Q(X)]}.
\end{equation}
Moreover,
\begin{equation}
    \lim_{t\to\infty} \mathbb{E}_{X\sim p_t}[Q(X)] \;=\; Q_*.
\end{equation}
\end{lemma}
Define Kullback–Leibler (KL) divergence between two densities $p$ and $q$ as
\[
\KL(p\,\|\,q)
\;:=\;
\int_{\mathcal{X}} p(x)\,
\log\!\left(\frac{p(x)}{q(x)}\right)\,\pi(dx).
\]

The following theorem shows that the sequence $\{p_t\}_{t\ge0}$ converges to $p_*$ in KL divergence and also converges uniformly on the maximizing set $A$.

\begin{theorem}[Convergence under Regimes (i) and (ii)] 
\label{thm:KL-uniform-RegimeI}
Recall $p_*$ in \eqref{equ:limit_dist_pure_syn}. Assume the Retraining Assumptions~\ref{re_ass:1_Shannon-entropy}–\ref{re_ass:3_ess-bdd-Q}, and Assumption~\ref{ass1:non0-mass}. If Regime (i) holds together with  Setting~\ref{setting:variable-noise}, or if Regime (ii) holds together with Setting~\ref{setting:process-noise}, then $\KL\bigl(p_*\,\|\,p_t\bigr)$ is non-increasing with $t$, strictly decreasing whenever $p_t=p_*$ and
\begin{equation}
\lim_{t\to \infty}\KL\bigl(p_*\,\|\,p_t\bigr)= 0 \text{ and } \lim_{t\to \infty } \sup_{\substack{x\in A\\ p_0(x)\neq 0}}\Bigl|\frac{p_t(x)}{p_*(x)}-1\Bigr| =0.
\end{equation}

\end{theorem}
Theorem \ref{thm:KL-uniform-RegimeI} generalizes Theorem 2.1 in  \cite{ferbach2024self} by allowing possibly nonstationary noise. Although the conclusions are the same as that of \cite{ferbach2024self}, technical difficulties arise as one need to find a substitution for Lemma 2.2 in \cite{ferbach2024self}, which does not hold in our settings.

\subsection{Convergence in retraining dynamics with regularization}
Because the convergence analyses in Regimes (iii) and (iv) are fundamentally different, we state them separately. Our setting covers the setting of \cite{ferbach2024self}, where they only considered deterministic reward function without human-preference fluctuation. More importantly, convergence result in the general nonparametric setting was not established in the prior work, with only a compromised exception under the parametric setting.

\subsubsection{Convergence analysis in Regime (iii)}
Under Regime (iii), we define nonlinear operator $T:\mathcal{P}_{\pi}\to\mathcal{P}_{\pi},$ 
\begin{equation}\label{equ:non_linear_oper}
   Tp(x)= \alpha p_{\rm ref}(x) + (1-\alpha)S_p(x) 
\end{equation}
where
\[
 S_p(x):=p(x)\,H_p^K(x).
\]
Thus we obtain the discrete-time dynamical system $p_{t+1}=T(p_t)$ with solution $p_t=T^{t}p_0$ for $t\in\mathbb{N}$. To establish existence and convergence of the limit, it is natural to invoke the Banach fixed point (contraction mapping) theorem. We now study sufficient conditions under which $T$ is a contraction mapping.

For any $B\subset \mathcal{X}$, define probability measure
\[
\mathbb{S}_p(B)=\int_B S_p(x)\pi(dx).
\]
Applying Lemma~\ref{lem:curated-distribution} with \(f=\mathbf{1}_B\) for any measurable set \(B\), we re‑express \(\mathbb{S}_p(B)\), yielding the following corollary.

\begin{corollary}\label{cor:selection-identity}
Under Setting~\ref{setting:variable-noise}, if \(X_1,\ldots,X_K\stackrel{i.i.d.}{\sim} p\in \mathcal{P}_\pi\) and set \(\tilde{p}_j=\tilde{p}_j(X_{1:K},\varepsilon_{1:K})=\dfrac{e^{r(X_j)+\varepsilon_j}}{\sum_{i=1}^K e^{r(X_i )+\varepsilon_i }}\),
then for all measurable \(B\in \mathcal{X}\),
\begin{equation}\label{equ:S_p_measure}
    \mathbb{S}_p(B)=\mathbb E\left[\sum_{j=1}^K \tilde{p}_j\,\mathbf{1}_B(X_j) \right].
\end{equation}
\end{corollary}
%As a direct consequence of \eqref{equ:S_p_measure}, we obtain the following sufficient condition for $T$ to be a contraction mapping on $(\mathcal{P}_\pi,d(\cdot,\cdot)_{\rm TV})$, where .

As a consequence of \eqref{equ:S_p_measure}, we obtain in Lemma \ref{lem:contraction} below a sufficient condition
for \(T\) to be a contraction on the complete metric space \((\mathcal{P}_\pi,d_{\rm TV}  )\). Here \(d_{\rm TV}(\cdot,\cdot) \) denotes the TV metric: for \(\mu,\nu\in\mathcal{P}_\pi\) with densities \(p=\frac{d\mu}{d\pi}\) and
\(q=\frac{d\nu}{d\pi}\),
\[
d_{\rm TV}(\mu,\nu)
\;=\;
\sup_{B}|\mu(B)-\nu(B)|
\;=\;\frac12\int_{\mathcal{X}} |p(x)-q(x)|\,\pi(dx).
\]
Henceforth, identifying measures with their densities w.r.t.~\(\pi\), we write $d_{\rm TV}(p,q)=d_{\rm TV}(\mu,\nu)$.

\begin{lemma}\label{lem:contraction}
    If $w,u\in \mathcal{P}_\pi$, then 
\[
d_{\rm TV}( \mathbb{S}_w,\mathbb{S}_u)\leq K\cdot d_{\rm TV}(w,u).
\]
    Hence,
\begin{equation}\label{equ:T_contraction}
    d_{\rm TV}(Tw,Tu)\leq (1-\alpha)K\cdot d_{\rm TV}(w,u),
\end{equation}
so $T$ is a contraction with respect to $d_{\rm TV}$ whenever $\alpha>\frac{K-1}{K}$.
\end{lemma}

\begin{theorem}[Geometric Convergence under Regime~(iii)]
\label{thm:tv-convergence-iii}
Under Regime~(iii), Setting~\ref{setting:variable-noise} and Retraining Assumptions~\ref{re_ass:1_Shannon-entropy}, let $2\le K<\infty$ and $\alpha\in\big( \frac{K-1}{K},\,1\big)$. Then $T$ is a contraction on $(\mathcal{P}_\pi,d_{\rm TV})$ with Lipschitz constant $\rho \;=\; K(1-\alpha)\in[0,1),$ 
and hence there exists a unique fixed point $p_*\in\mathcal{P}_\pi$ that satisfies $p_*=T(p_*)$.
Moreover, for every $p_0\in\mathcal{P}_\pi$, the TV metric $d_{\rm TV}(p_t,p_*)$ is non-increasing with $t$, and the iterates $p_{t+1}=T(p_t)$ converge geometrically to $p_*$ with
\[
d_{\rm TV}(p_t,p_*)\;\le\; \rho^{\,t}\, d_{\rm TV}(p_0,p_*), \qquad t\in\mathbb{N}.
\]
\end{theorem}
Intuitively, the more alternatives you compare ($K$ larger), the more reference mass $\alpha$ you must inject to keep the loop stable. Without extra assumptions, one can not hope to have an explicit characterization of the limit $p_*$ in this setting. We leave the problem of characterizing the limit $p_*$ as a potential future work.

\begin{remark}\label{rem:Regime_iii}
Theorem~\ref{thm:tv-convergence-iii} does not require the Retraining Assumption~\ref{re_ass:2_exp-intgr}; in particular, it imposes neither a boundedness assumption on \(r(x)\) nor an exponential-moment assumption on \(\varepsilon\). The theorem applies whenever \(\alpha \in \bigl( \frac{K-1}{K}  ,\,1\bigr)\). Under this condition, there exists a unique fixed point \(p_*\), and the iterates converge to \(p_*\) in the TV metric at a geometric rate. Moreover, the limit \(p_*\) depends only on \(p_{\rm ref}\), \(\alpha\) and the human curation procedure.

In particular, when \(p_{\rm ref}=p_0\) (the pretrained generative model used in applications), the population‑level retraining dynamics converge at a geometric rate, thereby providing a theoretical foundation for practical applications of RLHF (\citet{stiennon2020learning,nakano2021webgpt,ouyang2022training,touvron2023llama}).
\end{remark}

\begin{lemma}\label{lem:r_iii_reward_increase}
Under the same setting of Theorem~\ref{thm:tv-convergence-iii}, if $p_{\rm ref}=p_0$ and $\E_{p_0} Q(X)<Q^*$, then the expected exponential reward 
\begin{equation}\label{ineq:49_need}
    \E_{X\sim p_t} Q(X)\geq \E_{X\sim p_0} Q(X)+   (1-\alpha)^t  \operatorname{Cov}_{X\sim p_0}(Q(X),H_{p_0}^K(X))
\end{equation}
where $\operatorname{Cov}_{X\sim p_0}(Q(X),H_{p_0}^K(X))>0$. Moreover, \(\lim_{t\to\infty } \mathbb{E}_{X\sim p_t} Q(X) >\mathbb{E}_{X\sim p_0} Q(X).\)
\end{lemma}

\begin{remark}
Under the condition $\alpha \in \bigl(\frac{K-1}{K},1\bigr)$, Theorem~\ref{thm:tv-convergence-iii} establishes the convergence of the retraining dynamics in Regime~(iii). 
Notably, this result complements Theorem~2.4 in \citet{ferbach2024self}, which studies a similar iterative process under the condition $\lambda < \frac{1}{K-1}$—equivalently, $\alpha \in \bigl(\frac{K-1}{K},1\bigr)$—but does not establish convergence in that setting. 
Our theorem thus provides a refinement of their framework by offering a convergence guarantee under comparable assumptions. However, when \(\alpha \in (0,\frac{K-1}{K} ]\) where more human-curated data incur a more unstable update dynamic, it remains an open question whether the retraining dynamics converge; we leave the analysis of this case in Regime (iii) to future work.
\end{remark}

\subsubsection{Convergence analysis in Regime (iv)}
In Regime (iv), an infinite candidate pool ($K=\infty$) stabilizes the human‑curation procedure and allows us to relax the constraint on the regularization proportion $\alpha$ (Assumption~\ref{ass:A3}) relative to Regime (iii) and analyze convergence. We derive a closed-form expression for the fixed point of the nonlinear operator \(T\). 

Our analysis regarding Regime (iv) requires a technical tool from functional analysis: the Hilbert projective metric; see Definition \ref{Hilbert proj} below and \cite{atar1997exponential,eveson1995applications,eveson1995elementary} for basic properties and applications. A short explanation of why this metric is useful for the analysis is given in Remark \ref{remark Hilbert proj}. Roughly speaking, the nonlinear mapping we construct is not contractive in the TV metric, but it is contractive in the Hilbert projective metric.

For convenience, we consider a new baseline measure in place of $\pi$. Define the probability measure $\mathbb{P}_{\rm ref }$ such that $\mathbb{P}_{\rm ref }\ll \pi$ and $p_{\rm ref}= \frac{d \mathbb{P}_{\rm ref }}{d \pi}$. Define the linear operator \(L\) on the positive cone
\begin{align} \label{cone}
\mathcal K \;=\;\bigl\{f\in L^{1}( \mathbb{P}_{\rm ref} ): f(x)>0\text{ for }  \mathbb{P}_{\rm ref}\text{-a.s.\ }x\bigr\}
\end{align}
by
\begin{equation}\label{eq:def-L}
  L[f](x)
  \;=\;
  \alpha\,\innerproduct{f}{Q}_{\mathrm{ref}}
  \;+\;
  (1-\alpha)\,f(x)\,Q(x),
\end{equation}
where 
\[
 \innerproduct{f}{Q}_{\mathrm{ref}}
  :=
  \int_{\mathcal X} f(x)\,Q(x) \mathbb{P}_{\mathrm{ref}}(dx)=\int_{\mathcal X} f(x)p_{\mathrm{ref}}(x)Q(x) \pi(dx).
\]
Define the normalized operator $L_N$ on the positive cone $\mathcal{K}$ as
\[
L_N[f](x) = \frac{L[f](x)}{\innerproduct{f}{Q}_{\rm ref}}.
\]
Define \( \mathcal{P}_{\mathbb{P}_{\rm ref}}=\bigl\{ \mu \text{ is a probability measure on } (\mathcal{X},\mathcal{B}) : \mu \ll \mathbb{P}_{\rm ref}\,\bigr\}.\) One can see that the mapping $L_N: \mathcal{P}_{\mathbb{P}_{\rm ref}}\to  \mathcal{P}_{\mathbb{P}_{\rm ref}}$ is nonlinear.

We make the following assumption on $p_0$ and $p_{\rm ref}$:
\begin{assumption}\label{ass:A2_new}
   Assume $\operatorname*{ess\,sup}_{x\in\mathcal X }\ \frac{p_0(x)}{p_{\rm ref}(x)} < \infty$.
\end{assumption}
Let \(w_{t}(x)=p_{t}(x)/p_{\mathrm{ref}}(x)\) for any $k\in \mathbb{N}$. Assumption \eqref{ass:A2_new} ensures that $w_0\in \mathcal{P}_{\mathbb{P}_{\rm ref}}$. It is straightforward to verify that each $w_{t}(x)\in \mathcal{P}_{\mathbb{P}_{\rm ref}}$ represents the probability density function (PDF) of dynamics under the new baseline measure $\mathbb{P}_{\rm ref }$ for all $t\in \mathbb{N}$. The following lemma establishes the equivalence between the operators \(L, L_N\) and the update dynamics.

\begin{lemma}\label{lem:dynamics_w}
The densities \(w_t\), defined with respect to \(\mathbb{P}_{\rm ref}\), satisfy the recursion
\begin{gather}
  L_N(w_t)(x)=\frac{L[w_t](x)}{\innerproduct{ w_t }{Q}_{\rm ref}}\text{ and}\\
  w_{t+1}(x)=\alpha + (1-\alpha)\frac{Q(x)\,w_t(x)}{\innerproduct{w_t}{Q}_{\rm ref}}=  L_N[w_t](x)\label{equ:update_w_Regime_iv}
%\label{equ:update_w_Regime_iv}.
\end{gather}
\end{lemma}
 The convergence problem for $\{p_t\}_{t=1}^\infty$ is then equivalent to analyzing the fixed point for $L_N$. If the fixed point of $L_N$ exists, denoted as $w_{*}$, then
\[
L_N[w_*](x)=\alpha+(1-\alpha)\frac{Q(x)w_{*}(x)}{c_*}=w_{*}(x),
\]
where constant $c_*=\innerproduct{w_*}{Q}_{\rm ref}$. The fixed point must satisfy
\begin{equation}\label{equ:w_*}
    w_*(x)= \frac{\alpha}{1-(1-\alpha)Q(x)/c_*}.
\end{equation}
To make sure that the formal fixed point \eqref{equ:w_*} defines a proper density function, we introduce the following assumption.

\begin{assumption}[Nondegenerate fixed point condition]\label{ass:A3}
Assume that
\begin{equation}\label{equ:ass2_alpha}
  \alpha \;>\; \Bigg(\int \frac{Q_*}{\,Q_* - Q(x)\,}\, p_{\mathrm{ref}}(x)\, d\pi(x)\Bigg)^{-1}.
\end{equation}
\end{assumption}

\begin{remark}

By Retraining Assumption \ref{re_ass:3_ess-bdd-Q} and $p_{\rm ref}=p_0$, if $\mathbb{P}_{\rm ref}(A)>0$, it is obvious that
\[
\int \frac{Q_*}{\,Q_* - Q(x)\,}\, p_{\mathrm{ref}}(x)\, d\pi(x)\;>\;1,
\]
and consequently the right-hand side of \eqref{equ:ass2_alpha} is strictly less than \(1\). If Assumption~\ref{ass1:non0-mass} holds, then
\[
\int \frac{Q_*}{\,Q_* - Q(x)\,}\, p_{\mathrm{ref}}(x)\, d\pi(x)=\infty,
\]
so the requirement \eqref{equ:ass2_alpha} reduces to \(\alpha>0\); that is, there is no additional restriction on \(\alpha\) in this case.
\end{remark}
%------------------------------------------------------------------
%  Fixed‐point construction
%------------------------------------------------------------------

%Under Regime (iv), Setting \ref{setting:process-noise}, Retraining Assumption \ref{re_ass:1_Shannon-entropy}–\ref{re_ass:3_ess-bdd-Q}, and Assumption \ref{ass1:non0-mass},

%\begin{lemma}\label{lem:fixed_point_c}
 %   Under Regime (iv), Setting \ref{setting:process-noise}, Retraining Assumption~\ref{re_ass:3_ess-bdd-Q} and Assumption~\ref{ass:A3}, there exists a unique constant $c_*\in ((1-\alpha)Q_{*},Q_{*}]$ such that the $w_*$ defined in \eqref{equ:w_*} is the unique fixed point of $L_N$ over $\mathcal{P}_{\mathbb{P}_{\rm ref}}$.
%\end{lemma}

\begin{lemma}\label{lem:fixed_point_c}
Under Regime~(iv), Setting~\ref{setting:process-noise}, Retraining Assumption~\ref{re_ass:3_ess-bdd-Q}, and Assumption~\ref{ass:A3}, 
there exists a unique fixed point \( w_* \) of the operator \( L_N \) over \( \mathcal{P}_{\mathbb{P}_{\mathrm{ref}}} \). Moreover, the corresponding constant 
\( c_* \) lies in the interval 
\(((1-\alpha)Q_*,\, Q_*]\).
\end{lemma}

Nonlinear Perron–Frobenius theory offers a contraction framework for positive, order-preserving homogeneous maps, guaranteeing a unique positive fixed point and geometric convergence of normalized iterates \citep{lemmens2012nonlinear,eveson1995applications}. Because contraction is measured in the Hilbert projective metric, this metric is the natural geometry for analyzing the stability and convergence of our Regime (iv) retraining dynamics. Next we define the Hilbert projective metric on the positive cone $\mathcal K$ (defined in Subsection~\ref{subsec:notation}) of the underlying probability space.

\begin{definition}[Hilbert projective metric] \label{Hilbert proj}
Recall the cone $\mathcal{K}$ in \eqref{cone}. For $u,v\in \mathcal K$ with $u,v>0$ a.s., define
\[
\beta(u,v)\;:=\;\operatorname*{ess\,sup}_{x\in \mathcal{X}}\frac{u(x)}{v(x)},
\qquad
\alpha(u,v)\;:=\;\operatorname*{ess\,inf}_{x\in \mathcal{X}}\frac{u(x)}{v(x)} \;=\; \frac{1}{\beta(v,u)}.
\]
The Hilbert projective metric \citep{atar1997exponential,lemmens2012nonlinear} is
\[
d_{\mathcal{H}}(u,v)\;:=\;\log\!\frac{\beta(u,v)}{\alpha(u,v)}
\;=\;\log\!\big(\beta(u,v)\,\beta(v,u)\big).
\]
\end{definition}

Note that the definition of the Hilbert projective metric above coincides with the classical definition involving convex order in a Banach space; see \cite{lemmens2012nonlinear} for more details. We now motivate the use of the Hilbert projective metric in Regime (iv).
\begin{remark} \label{remark Hilbert proj}
In Regime (iii), the nonlinear operator \(T\) is not necessarily nonexpansive; indeed,
Theorem~\ref{thm:tv-convergence-iii} shows nonexpansiveness only when
\(\alpha\in\bigl[ \frac{K-1}{K},\,1\bigr)\).
In Regime (iv), the update \(L_N[w]=L[w]/\innerproduct{w}{Q}_{\mathrm{ref}}\) is the projective normalization of the positive linear map $L[f]$. Because the Hilbert projective metric \(d_\mathcal{H}\) is invariant under positive scalings, we have
\(d_\mathcal{H}\bigl(L_N[u],L_N[v]\bigr)=d_\mathcal{H}\bigl(L[u],L[v]\bigr)\).
By Birkhoff’s contraction theorem~\citep{eveson1995applications,eveson1995elementary},
\[
d_\mathcal{H}(w_{t+1},w_*)
\;=\;
d_\mathcal{H}\bigl(L_N[w_t],L_N[w_*]\bigr)
\;=\;
d_\mathcal{H}\bigl(L[w_t],L[w_*]\bigr)
\;\le\;
d_\mathcal{H}(w_t,w_*),\ t\ge 0.
\]
%In particular, \(d_\mathcal{H}(w_{k+1},w_*)\le d_\mathcal{H}(w_k,w_*)\).
By contrast, other metrics (e.g., the TV metric or the $KL$ divergence) are not guaranteed to yield monotonic convergence of $L_N^t[w_0]$ toward $w_*$. Because convergence in the Hilbert projective metric is stronger, once established, it directly yields convergence in the TV metric and KL divergence.

%Moreover, both the KL divergence and the total variation distance are dominated by the Hilbert projective metric, allowing us to establish even stronger convergence results.

%of the metric to $w_*$ under iteration of $L_N$. What's more, both KL divergence and total variation distance are controlled by the Hilbert projective metric, allowing us to build even stronger convergence result.

\end{remark}

\begin{theorem}\label{thm:Regime_iv_convergence}
 Under Regime (iv), Setting \ref{setting:process-noise}, Retraining Assumption~\ref{re_ass:1_Shannon-entropy}-\ref{re_ass:3_ess-bdd-Q} and Assumptions~\ref{ass:A2_new}-\ref{ass:A3}, if
\begin{equation}\label{equ:w_0_upper}
    Q_{\min}:=\operatorname*{ess\,inf}_{x\in\mathcal X }\  Q(x)>0,
\end{equation}
then the Hilbert projective metric $d_{\mathcal{H}}(w_{t},w_*)$ is non-increasing and remains finite for all $t\geq 1$
%    If $R=d_{\mathcal{H}}(w_0,w_*)\in (0,\infty)$, then 
and
\begin{equation*}
    \lim_{t\to \infty}d_{\mathcal{H}}(w_{t},w_*)=0.
\end{equation*}
\end{theorem}
Under the original dynamics \eqref{equ:update_Regime_iii_iv} and $p_{\rm ref}=p_0$, the result can be rewritten as
\[
\lim_{t\to \infty}d_{\mathcal{H}}(p_{t},p_*)=0,
\]
where $p_t=w_tp_{0}$ and $p_*=w_*p_{0}$. Clearly the limit $p_*$ depends on the pre-trained generative model $p_0$.
\begin{remark}
If we set $\varepsilon\equiv 0$, our setting reduces to that of \citet{ferbach2024self}. {However, the results in Theorem 2.4 of \citet{ferbach2024self}  does not give any convergence result, whereas our Theorems~\eqref{thm:Regime_iv_convergence} not only gives convergence in the KL divergence and the TV metric, but also the stronger the Hilbert projective metric.}
\end{remark}

\begin{lemma}\label{lem:r_iv_reward_lbound}
Under the same setting as Theorem~\ref{thm:Regime_iv_convergence},
suppose that \(p_{\rm ref}=p_0\) and \(\mathbb{E}_{X\sim p_0}[Q(X)]\in(0,Q_*).\) Then, for every integer \(t\geq 1\), \(\mathbb{E}_{X\sim p_t}[Q(X)]
>
\mathbb{E}_{X\sim p_0}[Q(X)].\) More precisely,
\begin{equation}
\mathbb{E}_{X\sim p_t}[Q(X)]
\geq
\mathbb{E}_{X\sim p_0}[Q(X)]
+
\frac{1-\alpha}{Q_*}
\left(1-(1-\alpha)^t\right)
\operatorname{Var}_{p_0}(Q(X)).
\label{eq:regime-iv-reward-lower-bound}
\end{equation}
\end{lemma}

\section{Stability under Reward Perturbations}\label{subsec:unstable}
At each retraining step we form a convex mixture between a fixed reference distribution $p_{\rm ref}$ (e.g., real data or an anchored model) and a curated synthetic distribution produced from the current model $p$. We study the sensitivity of this update to bounded reward perturbations. Concretely, we perturb the per-sample reward by an additive $\Delta r\in L^\infty(\mathcal X,\mathcal B,\pi)$, with $\|\Delta r\|_\infty\le\eta$, and ask: \emph{how far can the entire retraining trajectory drift (in the TV metric) from the unperturbed one, uniformly over time?} Related safety controls for generative models include content‑aware filtering and robustness‑improving training \citep{schramowski2023sld,wang2023better}.

\begin{definition}[Superalignment]
\label{def:superalignment}
Let $(p_{t,\Delta r})_{t\in\mathbb{N}}$ be retraining dynamics under reward $r+\Delta r$, with $\Delta r\in L^\infty(\mathcal X,\mathcal B,\pi)$ and $\|\Delta r\|_\infty\le\eta$. We say \emph{superalignment} holds if:
\begin{enumerate}[label=(\roman*),leftmargin=*]
\item \textbf{Improvement:} 
$\inf_{t\geq 1}\displaystyle \mathbb{E}_{X\sim p_{t }}[Q(X)]>\mathbb{E}_{X\sim p_{0 }}[Q(X)]$.
\item \textbf{Trajectory stability:} 
$\displaystyle \limsup_{\|\Delta r\|_\infty\to 0}\;\sup_{t\in\mathbb N\cup\{\infty\}} d_{\mathrm{TV}}\!\bigl(p_{t,\Delta r},p_{t,0}\bigr)=0$.
\end{enumerate}
\end{definition}

\subsection{Regimes (i) and (ii): Unstable under bounded reward perturbations}
Throughout this subsection we work under the assumptions
of our convergence results in the pure synthetic setting: \ref{re_ass:1_Shannon-entropy}-\ref{re_ass:3_ess-bdd-Q} and Assumption~\ref{ass1:non0-mass} for
Regime~(i) ($\alpha=0$, $K<\infty$) under Setting~\ref{setting:variable-noise}, and for Regime~(ii) ($\alpha=0$, $K=\infty$)
under Setting~\ref{setting:process-noise}. In both regimes, one retraining step with a
perturbed reward $r+\Delta r$ is
\[
p_{t+1,\Delta r}(x) \;=\; p_{t,\Delta r}(x)\;H^{K}_{\,p_{t,\Delta r} }(x),
\]
where for $2\leq K<\infty$
\[
    H_{p,\Delta r}^{K}(x):=
      \E_{\substack{X_{1},\dots,X_{K-1}\sim p\\
       \varepsilon,\varepsilon_{1 },\ldots,\varepsilon_{K-1 } }} \Bigl[ K\,\frac{e^{\,r(x)+\Delta r(x)+\varepsilon(x)}}{ e^{\,r(x)+\Delta r(x)+\varepsilon(x)}+ \sum_{k=1}^{K-1}e^{r(X_k)+\Delta r(X_k)+\varepsilon_k(X_k)}}
      \Bigr]
\]
and for $K=\infty$
\[
    H_{p,\Delta r}^{\infty}(x):= \frac{ Q_{\Delta r}(x)}{ \E_{X\sim p} Q_{\Delta r}(X)  }.
\]
where $Q_{\Delta r}(x):=e^{r(x)+\Delta r(x)}\mathbb{E}\left[e^{\varepsilon(X)}\mid X=x\right]=Q(x)e^{ \Delta r(x)}$, $Q_{*,\Delta r}:=\operatorname*{ess\,sup}_x Q_{\Delta r}(x)$ and
\(
A_{\Delta r}:=\{x\in \mathcal{X}:\;Q_{\Delta r}(x)=Q_{*,\Delta r}\}.
\)
Then the population limit is the restriction of $p_0$ to the maximizing level set:
\begin{equation}\label{eq:limit-measure}
p_{\infty,\Delta r}(x)
\;=\;
\frac{p_0(x)\,\mathbf 1_{A_{\Delta r}}(x)}{\mathbb{P}_0(A_{\Delta r})},
\ \mathbb{P}_0(A_{\Delta r})>0,
\end{equation}
and analogously for $\Delta r\equiv 0$ with $A=A_{0}$; see Theorem~\ref{thm:KL-uniform-RegimeI}.

\begin{theorem} 
\label{thm:TV_instability}
For fixed $\eta>0$, there exists $\delta>0$ such that 
\[
\mathbb{P}_0\left(\{x\in \mathcal{X}:\;Q_*-\delta\le Q(x)<Q_*\}\right)>0
\quad\text{and}\quad
(Q_*-\delta)\,e^{\eta}\;\ge\;Q_*.
\]
Then there exists a measurable $\Delta r$ with $\|\Delta r\|_{\infty}=\eta$ such that $\mathbb{P}_0(A_{\Delta r})>0$ and
\[
d_{\mathrm{TV}}\left(p_{\infty,\Delta r},\,p_{\infty,0}\right)=1.
\]
\end{theorem}
The limit of fully synthetic retraining dynamics in Regimes (i) and (ii) are {unstable} in total
variation under arbitrarily small $L^\infty$ reward perturbations.

\subsection{Regime (iii): Lipschitz-stable updates under bounded reward perturbations}
\label{subsec:regime-iii}

Define perturbed update map
\begin{equation}
    T_{\Delta r} p(x) = \alpha p_{\rm ref}(x) +(1-\alpha)   p(x)H^K_{p,\Delta r}(x),
\end{equation}
where $K\geq 2$, and $\alpha \in(\frac{K-1}{K},1)$. When $\Delta r=0$, $T_{\Delta r}$ recovers nonlinear operator $T$ defined in \eqref{equ:non_linear_oper}.

Let
\[
T^\infty_{\Delta r}w_0=\lim_{n\to \infty}T^n_{\Delta r}w_0.
\]
Two forces act on the update: (i) the curation step (the $H_p^K$ factor) amplifies any shift in rewards; (ii) the mixture with $p_{\rm ref}$ damps those shifts by $\alpha$. After mixing, the contraction coefficient suggested by Theorem~\ref{thm:tv-convergence-iii} is \(\rho \;=\; (1-\alpha)K<1,\)
requiring $\alpha>\frac{K-1}{K}$. The next theorem shows uniform‑in‑iteration TV robustness to bounded reward perturbations.

\begin{theorem}\label{thm:stable_1}
Under Regime~(iii), Setting~\ref{setting:variable-noise} and Retraining Assumptions~\ref{re_ass:1_Shannon-entropy}, let $2\le K<\infty$ and $\alpha\in\big( \frac{K-1}{K},\,1\big)$. We have
\begin{equation}\label{lim_thm:stability_11_general}
   \sup_{ \|\Delta r\|_\infty\leq \eta }\sup_{n\in \mathbb{N}\cup \{\infty\} } d_{\rm TV}(T^n_{\Delta r} p_0,T^n_{0} p_0) \leq \frac{ \eta \rho }{4(1-\rho)}.
\end{equation}

\end{theorem}

The supremum over $n\in\mathbb{N}\cup\{\infty\}$ guarantees that no matter how long you retrain, bounded reward noise cannot push the trajectory farther than the right‑hand side of \eqref{lim_thm:stability_11_general}.

\subsection{Regime (iv): Stable updates under bounded reward perturbations}
\label{subsec:regime-iv}

Define perturbed dynamic mapping
\begin{equation}\label{equ:dynamic_L_regime_iv}
    \mathfrak{L}_{\Delta r} w(x) = \alpha +(1-\alpha) \frac{Q_{\Delta r}(x)\,w(x)}{\innerproduct{w}{Q_{\Delta r}}_{\rm ref}},
\end{equation}
where $Q_{\Delta r}(x):=e^{r(x)+\Delta r(x)}\mathbb{E}\left[e^{\varepsilon(X)}\mid X=x\right]=Q(x)e^{ \Delta r(x)}$, $\Delta r\in L^\infty(\mathcal X,\mathcal B,\pi)$ and $w\in (\mathcal{P}_{\rm ref},d_{\rm TV})$. Let
\[
\mathfrak{L}^\infty_{\Delta r}w_0=\lim_{n\to \infty}\mathfrak{L}^n_{\Delta r}w_0.
\]

\begin{assumption}\label{ass:4_perturbation}
    For fixed $\eta_*>0$, set $Q_{*,\Delta r}=\mathbb P_0\operatorname*{-ess\,sup}_{x\in \mathcal{X}} Q_{\Delta r}(x)$. Assume
    \begin{equation}\label{equ:Ass:A6}
    \alpha> \sup_{\|\Delta r \|_\infty\leq \eta_* }\Big( \int 
\frac{Q_{*,\Delta r}}{Q_{*,\Delta r} - Q_{\Delta r}(x)\,}\;p_{\mathrm{ref}}(x)\,d\pi(x)  \Big)^{-1}.
\end{equation}

\end{assumption}

\begin{theorem}\label{thm:stable_2}
    Under Regime (iv), Setting \ref{setting:process-noise}, Retraining Assumption \ref{re_ass:1_Shannon-entropy}-\ref{re_ass:3_ess-bdd-Q} and Assumptions~\ref{ass:A2_new} and \ref{ass:4_perturbation}, we have 
\begin{equation*} 
    \limsup_{\|\Delta r\|_{\infty}\to 0} \sup_{n\in \mathbb{N}\cup \{\infty\} } d_{\rm TV}(\mathfrak{L}^n_{\Delta r} w_0, \mathfrak{L}^n_{0} w_0) = 0.
\end{equation*}
\end{theorem}

Taken together, Theorems~\ref{thm:stable_1} and \ref{thm:stable_2} with Lemmas~\ref{lem:r_iii_reward_increase} and \ref{lem:r_iv_reward_lbound} imply that the regularized retraining dynamics achieve superalignment in the sense of Definition~\ref{def:superalignment} in Regimes (iii) and (iv). By contrast, in the unregularized Regimes (i) and (ii), the retraining process is unstable under bounded reward perturbations, as established in Theorem~\ref{thm:TV_instability}, despite exhibiting convergence and monotonic improvement in expected exponential reward. This observation underscores a practical lesson for curate-and-retrain loops in generative modeling: incorporating reference data during training is not merely convenient but essential for ensuring stability and robustness against reward misspecification noise \citep{bertrand2023stability,bohacek2023nepotistically,ruiz2023dreambooth,zhong2022deep,stiennon2020learning,hancock2019learning,wang2023self,he2016dual,lamb2016professor,zhu2017unpaired}. {Our analysis also complements recent findings on feedback loops and data contamination in self-consuming pipelines and on the role of constrained generative sampling \citep{ferbach2024self,hataya2023corrupt,martinez2023gaiInternet,kong2024constrained}.
}

\section{Discussion}

In this paper, we investigate the convergence behaviors and stability of the retraining dynamics under the four Regimes specified in Section \ref{sec:2_setup_}. Our results generalize some aspects of \cite{ferbach2024self} in Regimes (iii) and (iv), and under a more general model setting with individual fluctuations in the reward functions. We establish convergence and, as a consequence, prove robustness of the retraining dynamics to bounded reward perturbations, highlighting the necessity of reference data for stabilizing the retraining loop. Our analyses are based on a few new technical results, including the nonlinear Perron–Frobenius theory, and we expect them to be useful for other similar problems. Moreover, all results above extend verbatim to conditional generative models. Given a prompt $y$, replace in step~1 the unconditional sampling by
\[
X_1,\dots,X_K \stackrel{\text{i.i.d.}}{\sim} p_t(\cdot \mid y),
\]
and in step~2 replace the utility by the conditional counterpart,
\[
r(x \mid y) + \varepsilon_k(x \mid y).
\]
Equivalently, replace every occurrence of $X$ by the conditional variable $X \mid Y=y$ throughout the analysis. Under this substitution, all population-level updates and convergence/stability results in the paper hold without additional technical difficulty: the discrete-choice curation kernel, mean-field limit, and mixed-data contraction carry over pointwise in $y$, and integration over the prompt distribution preserves the proofs. 

%Our analysis deliberately avoids finite‑dimensional parametric restrictions on the generator $p_t$, the reference $p_{\rm ref}$, and the reward $r(\cdot)$. Apart from modeling curation with a Plackett–Luce choice rule, the only conditions are integrability and boundedness (Assumptions~\ref{re_ass:1_Shannon-entropy}–\ref{re_ass:3_ess-bdd-Q}). Consequently, the convergence and stability results apply in a distribution‑free manner to heterogeneous and noisy raters, in contrast to analyses that proceed within fixed parametric families or assume homogeneous, noise‑free selection \citep[e.g.,][]{ferbach2024self}.

Finally, we highlight several open questions that arise from our analyses.

\begin{itemize}
    \item \textbf{Sharpness of the convergence rate}. %{ Perhaps let's not use minimax? I think minimax usually refers to the testing/estimation rate type of things. How about just ``sharpness of the convergence rate"?}
    Is Theorem~\ref{thm:tv-convergence-iii} in Regime~(iii) sharp for uniform TV contraction? 
    In particular, can one characterize whether a change in convergence \emph{form} or \emph{rate} occurs at the threshold \(\alpha=(K-1)/K\)? 
    At this critical value, the worst‑case rate may cease to be geometric and exhibit only polynomial decay or even not convergence.

    \item \textbf{Necessity of stationarity in preference noise.}
    Do Theorem~\ref{thm:KL-uniform-RegimeI} (Regime~(i)) and Theorem~\ref{thm:tv-convergence-iii} (Regime~(iii)) continue to hold under the nonstationary preference‑noise model in Setting~\ref{setting:process-noise} (instead of Setting~\ref{setting:variable-noise})? 
    How robust are these guarantees as the heterogeneity of the noise field increases—for example, as 
    \(\frac{\operatorname{ess\,sup}_{x\in\mathcal{X}} Q(x)}{\operatorname{ess\,inf}_{x\in\mathcal{X}} Q(x)}\) grows? 
    It appears plausible that sufficiently heterogeneous noise can disrupt convergence in Regimes~(i) and~(iii) across standard metrics; identifying precise thresholds or explicit counterexamples is an open direction.

    \item \textbf{Effect of large comparison pools (\(K\)).}
    Intuitively, showing more candidates per round should help the curator select higher‑reward samples. 
    However, Theorem~\ref{thm:tv-convergence-iii} in Regime~(iii) requires a stronger lower bound on the mixing proportion, whereas in the infinite‑pool limit (Regime~(iv)) the sufficient condition on \(\alpha\) is much weaker (e.g., under Assumption~\ref{ass:A3}). 
    As \(K\) increases, the guarantees appear more resilient to heterogeneous preference noise, with \(K=\infty\) the most robust case; quantifying the precise dependence on \(K\) and the minimal mixing needed for stability and sharp rates remains open.
\end{itemize}

\bibliographystyle{apalike}  % or abbrvnat, unsrtnat, etc.
\bibliography{reference}      % no “.bib” extension

\appendix

\section{Auxiliary Results for Retraining Dynamics}

\begin{proof}[Proof of Lemma~\ref{lem:curated-distribution}]
Since $H^K_p(x) \in [0,K]$ and $K<\infty$, we have $$\E_{X\sim p}[|f(X)H_p^K(X)|]\leq K \E_{X\sim p}[|f(X)|]<\infty,$$
i.e., the right hand side of \eqref{equ:5_functional_int} is integrable. 

Because
$\Pr(\widehat X=X_k\mid X_{1:K},\varepsilon_{1:K})=\tilde{w}_k/\sum_j\tilde{w}_j$, we have
\[
  \E\bigl[f(\widehat X)\mid X_{1:K},\varepsilon_{1:K}\bigr]
  =\sum_{k=1}^K
     f(X_k)\,\frac{\tilde{w}_k}{\sum_{j=1}^K\tilde{w}_j}.
\]
Taking expectations and using exchangeability of the pairs $(X_k,\varepsilon_k)$ yields
\[ 
  \E[f(\widehat X)]=K\,\E\bigl[f(X_K)\tilde{w}_K/\sum_j\tilde{w}_j\bigr]=\E_{X\sim p}[f(X)H_p^K(X)].
\]
For any measurable set $A\subset \mathcal{X}$, taking $f(x)\equiv \mathbf{1}_A(x)$, which is the indicator function on $A$, we have
\begin{equation*}
    \mathbb{P}(\widehat{X}\in A) = \mathbb{E}[\mathbf{1}_A(\widehat X)]=\int_A H_p^K(x)p(x) \pi(dx),
\end{equation*}
so $pH_p^K$ is indeed the density of $\widehat X$.

\end{proof}

\begin{proof}[Proof of Theorem~\ref{thm:pure-Kfinite}]
Consider the following induction statements associate with induction index $m\in \mathbb{N}$.

\textbf{Statement 1:} For \(t = m-1\) the density \(p_{t+1}\) uniquely maximizes the pure self‑consuming objective \( p\longmapsto \mathcal L(p;p_t,K,0) =\mathbb E_{\substack{X_{1},\dots,X_{K}\sim p_t\\
      \widehat{X}\sim \mathcal{PL}(X_{1:K}, \varepsilon_{1:K})}} \bigl[\log p(\widehat X)\bigr].\) 

\textbf{Statement 2:} The Shannon entropy remains finite:
      \(
        |h(p_{m})|<\infty.
      \)

For the base case $m=0$, the Statement 1 is empty and Statement 2 is assumed in \ref{re_ass:1_Shannon-entropy}.

For induction step, we assume Statement 1 and Statement 2 hold at $m=k$.

Let \(q\) be the density of \(\widehat X\). Writing the objective in
integral form gives
\[
  \mathbb E \bigl[\log p(\widehat X)\bigr]
  \;=\;
  \int_{\mathcal X} q(x)\,\log p(x)\,\pi(dx).
\]
By induction Statement 2 with $m=k$, the $h(p_k)$ is finite. By Lemma~\ref{lem:curated-distribution}, we know that 
\[
\mathbb E_{\substack{X_{1},\dots,X_{K}\sim p_t\\
      \widehat{X}\sim \mathcal{PL}(X_{1:K}, \varepsilon_{1:K})}} \bigl[\log p(\widehat X)\bigr] = \E_{X\sim p_t} \bigl[\log p(X) H^K_{p_t}(X)\bigr].
\]
By Gibbs’ inequality this functional is uniquely maximized at \(p=q\).
Lemma~\ref{lem:curated-distribution} yields \(q(x)=p_{k}(x)H_{p_k}^{K}(x)\). Hence \(p_{k+1}=p_kH_{p_k}^{K}\) and the maximizer is unique, proving~Statement 1 with $m=k+1$.

Define \(h_t:=\int p_{t}(x)|\log p_{t}(x)|d\pi(x)\). By induction Statement 2 with $m=k$, we know that $0\leq h_k<\infty$.

Because \(0<H_{p_{k}}^{K}(x)\le K\) for all \(x\),
decompose
\[
  h_{k+1}
  \leq \int p_kH_{p_k}^{K}|\log p_k|\,d\pi
   +\int p_kH_{p_k}^{K}|\log H_{p_k}^{K} | \,d\pi
  =:A_k+B_k.
\]
Using \(H_{p_k}^{K}\le K\), we have
\( |A_k| \le K \int |p_k\log p_k|\,d\pi=K|h_k|   <\infty\).

\noindent For \(u\in(0,K]\) the function \(u\mapsto u|\log u|\) attains its
maximum at either \(u=e^{-1}\) or \(u=K\); hence
\(u|\log u|\le C_K:=\max\{e^{-1},K\log K\}\).
Consequently
\(   |B_k|  \le C_{K} \int p_k \,d\pi  =  C_{K}<\infty.\)
Thus \(h_{k+1}=A_k+B_k\leq K |h_k|+C_K <\infty\) proving~Statement 2 with $m=k+1$, which completes the induction step.

Hence, Statement 1 and Statement 2 both hold for all natural number $m$.
\end{proof}

\begin{proof}[Proof of Lemma~\ref{lem:Hlimit}]
Set \(Y_1(x):=e^{r(x)+\varepsilon(x)}\) and, for \(k\ge2\),
\(Y_k:=e^{r(X_k)+\varepsilon_k(X_k)}\), where the $X_k$ are i.i.d. with distribution $p$.  Write
\[
  \overline{Y}_{K}
  \;:=\;
  \frac1{K-1}\sum_{k=2}^{K}Y_k, 
  \qquad K\ge2.
\]
By the strong law of large numbers,
\(
  \overline{Y}_{K}\xrightarrow[K\to\infty]{\text{a.s.}}\E Y_2,
\)
so
\begin{equation}\label{equ:LLN_9}
     \frac{K\,Y_1(x)}{Y_1(x)+\sum_{k=2}^{K}Y_k}
  \;=\;
  \frac{Y_1(x)}{Y_1(x)/K+(K-1)\overline{Y}_{K}/K}
  \xrightarrow[K\to\infty]{\text{a.s.}}
  \frac{Y_1(x)}{\E Y_2}.
\end{equation}
Hence the pointwise limit of \eqref{equ:def_H_infty} holds once we justify an exchange of
limit and expectation.\\
Because \(1/u\) is convex on \((0,\infty)\), we have
\[
  \frac1{\overline{Y}_{K}}
  \;\le\;
  \frac1{K-1}\sum_{k=2}^{K}\frac1{Y_k},
\]
which implies that
\begin{equation}\label{ineq:9}
      \frac{K\,Y_1(x)}{Y_1(x)+\sum_{k=2}^{K}Y_k} \leq\frac{K}{K-1}Y_1(x) \frac{1}{\overline{Y}_{K}}
  \;\le\; 
  \frac{2Y_1(x)}{K-1}\sum_{k=2}^{K}\frac1{Y_k}=\frac2{K-1}\sum_{k=2}^{K}\frac{Y_1(x)}{Y_k}.
\end{equation}

The assumption
\(
  \E e^{\,|r(X)+\varepsilon(X)|}<\infty
\)
implies both
\(
  \E Y_2<\infty
\)
and
\(
  \E Y_2^{-1}<\infty.
\)
Consequently each random variable
\(Z_k:=Y_1(x)\,Y_k^{-1}\) has the same distribution and $\E |Z_k|<\infty$, and the family
\(\{Z_k:k\ge2\}\) is uniformly integrable.  

% Requires: amsmath, amssymb, amsthm
\begin{claim}[Averaging preserves uniform integrability]\label{lem:UI-averages}
Let $\{Z_k:k\ge 1\}$ be a uniformly integrable family.
For $K\ge 2$, define the averages
\[
\bar Z_K \;:=\; \frac{1}{K}\sum_{k=1}^K Z_k.
\]
Then the family $\{\bar Z_K:K\ge 1\}$ is uniformly integrable.
\end{claim}

\begin{proof}[Proof of Claim~\ref{lem:UI-averages}]
By the de la Vallée--Poussin criterion \citep{rao1991theory} for uniform integrability, there exist a nondecreasing convex function
$\Phi:[0,\infty)\to[0,\infty)$ with $\displaystyle\lim_{x\to\infty}\frac{\Phi(x)}{x}=\infty$ and a constant $C<\infty$ such that
\[
\sup_{k\ge 1}\;\mathbb E\big[\Phi(|Z_k|)\big]\;\le\; C.
\]
For fixed $K\ge 2$, we have
\[
\Phi\!\left(|\bar Z_K|\right)
\;=\;\Phi\!\left(\Big|\frac{1}{K}\sum_{k=1}^K Z_k\Big|\right)
\;\le\;\Phi\!\left(\frac{1}{K}\sum_{k=1}^K |Z_k|\right)
\;\le\;\frac{1}{K}\sum_{k=1}^K \Phi\!\left(|Z_k|\right),
\]
where the last inequality is Jensen’s inequality applied to the convex $\Phi$.
Taking expectations and using the uniform bound,
\[
\mathbb E\big[\Phi(|\bar Z_K|)\big]
\;\le\;\frac{1}{K}\sum_{k=1}^K \mathbb E\big[\Phi(|Z_k|)\big]
\;\le\; C
\quad\text{for all }K\ge 2.
\]
Therefore $\sup_{K\ge 2}\mathbb E\big[\Phi(|\bar Z_K|)\big]\le C<\infty$, and by the de la Vallée--Poussin criterion again,
the family $\{\bar Z_K:K\ge 2\}$ is uniformly integrable.
\end{proof}

By Claim \ref{lem:UI-averages}, the sequence defined by the right-hand side of \eqref{ineq:9} is uniformly integrable. Hence, the family
\(
  \bigl\{\tfrac{K\,Y_1(x)}{Y_1(x)+\sum_{k=2}^{K}Y_k}:K\ge2\bigr\}
\)
is dominated by a uniformly integrable family $\{\frac{2Y_1(x)}{K-1}\sum_{k=2}^{K}\frac1{Y_k}:K\geq 2 \} $ and is therefore uniformly integrable itself.

Uniform integrability, combined with the almost-sure convergence \eqref{equ:LLN_9}, allows us to pass the limit inside the expectation, yielding the claimed
identity
\(
  H_{p}^{\infty}(x)=\lim_{K\to\infty}H_{p}^{K}(x)=\frac{Q(x)}{\mathbb{E}_{X\sim {p}}Q(X)}.
\)
\end{proof}

\begin{proof}[Proof of Theorem~\ref{thm:pure-Kinfty}]Consider the following induction statements associated with induction index $m\in \mathbb{N}$.

\textbf{Statement 1:} For \(t = m-1\) the density \(p_{t+1}\) uniquely maximizes the pure self‑consuming objective \eqref{eq:inf-K-objective}, \( p\longmapsto \mathcal L(p;p_t,\infty,0) =\mathbb E_{X\sim p_t}\bigl[\log p(X)H_{p_t}^{\infty}(X)\bigr].\) 

\textbf{Statement 2:} The Shannon entropy remains finite:
      \(|h(p_{m})|<\infty.\) 

\textbf{Statement 3:} The exponential reward satisfies $\E_{X\sim p_{m}} Q(X) \geq \E_{X\sim p_{0}} Q(X)>0$. 
      
For the base case $m=0$, the Statement 1 is empty, Statement 2 is assumed in \ref{re_ass:1_Shannon-entropy} and Statement 3 is trivial.

For induction step, we assume Statement 1 and Statement 2 hold at $m=k$.

Writing the objective in
integral form gives
\[
  \mathbb E \bigl[\log p(x)H_{p_k}^{\infty}(x)\bigr]
  \;=\;
  \int_{\mathcal X} p_k(x)\,\log p(x)H_{p_k}^{\infty}(x)\,\pi(dx).
\]
By induction Statement 2 with $m=k$, the $h(p_k)$ is finite. It is easy to verify that $p_kH_{p_k}^\infty$ is a density (w.r.t. $\pi$). By Gibbs’ inequality this functional is uniquely maximized at \(p=p_kH_{p_k}^\infty\). Hence \(p_{k+1}=p_kH_{p_k}^{\infty}\) and the maximizer is unique, proving~Statement 1 with $m=k+1$.

Define \(h_t:=\int p_{t}(x)|\log p_{t}(x)|d\pi(x)\). By induction Statement 2 with $m=k$, we know that $0\leq h_k<\infty$.

By induction Statement 3 with $m=k$, we know that \(0<H_{p_{k}}^{\infty}(x)\le Q_*/\E_{X\sim p_k}Q(X)\leq Q_*/\E_{X\sim p_0}Q(X)=:D<\infty \) $\pi$-a.s..
Decompose
\[
  h_{k+1}
  = \int p_kH_{p_k}^{\infty}|\log p_k|\,d\pi
   +\int p_kH_{p_k}^{\infty}|\log H_{p_k}^{\infty} | \,d\pi
  =:A_k+B_k.
\]
We have
\( |A_k| \le D \int |p_k\log p_k|\,d\pi=D|h_k|   <\infty\).  For \(u\in(0,D]\), the function \(u\mapsto u|\log u|\) attains its
maximum at either \(u=e^{-1}\) or \(u=D\); hence
\(u|\log u|\le C_D:=\max\{e^{-1},D\log D\}\).
Consequently
\(   |B_k|  \le C_{D} \int p_k \,d\pi  =  C_{D}<\infty.\)
Thus, \(h_{k+1}=A_k+B_k\leq D |h_k|+C_D <\infty\), proving~Statement 2 with $m=k+1$

Since we already verified Statement 1 with $m=k+1$, we know that
\[
\E_{X\sim p_{k+1}}Q(X)-\E_{X\sim p_{k}}Q(X) = \frac{\E_{X\sim p_k}Q^2(X) - (\E_{X\sim p_k}Q(X))^2 }{\E_{X\sim p_k}Q(X)}\geq 0.
\]
Thus, Statement 1 with $m=k$ holding implies the Statement 1 with $m=k+1$ also holds, which completes the induction step.

Hence, Statements 1-3 hold for all natural numbers $m$.

\end{proof}

\begin{proof}[Proof of Theorem~\ref{thm:update-mix}]

    Write $\mathcal{L}(p;p_t,K,\alpha)=\int\psi(x)\log p(x)\pi(dx)$, where $\psi=\alpha p_{\rm{ref}}+(1-\alpha) H_{p_t}^Kp_t$ is a PDF by Lemma \ref{lem:curated-distribution}. Following the proofs of Theorems~\ref{thm:pure-Kfinite} and \ref{thm:pure-Kinfty}, if $\psi$ has finite Shannon entropy, the result of Theorem~\ref{thm:update-mix} follows from the Gibbs' inequality immediately. It suffices to show that finiteness of the Shannon entropy is preserved under convex combinations of densities.

    \begin{claim}[Convex combination preserved finite entropy]
        Let $p_1,p_2\in\Prob_\pi$ have finite Shannon entropies.  Fix any combination weight \(\alpha\in(0,1)\) and set \( p_{\alpha}(x):=\alpha\,p_{1}(x)\;+\;(1-\alpha)\,p_{2}(x), x\in\mathcal X.\) Then \(p_{\alpha}\in\Prob_\pi\) and has finite Shannon entropy.
    \end{claim}

    \begin{proof}
Convexity of the function \(\phi(u):=u\log u\) on \((0,\infty)\) implies,
for every \(x\in\mathcal X\),
\begin{equation}\label{eq:phi-upper}
      \phi\bigl(p_{\alpha}(x)\bigr)
  \;\le\;
  \alpha\,\phi\bigl(p_{1}(x)\bigr)
  +(1-\alpha)\,\phi\bigl(p_{2}(x)\bigr).
\end{equation}
It is straightforward that
\begin{equation}\label{eq:phi-lower}
\begin{aligned}
    \phi\bigl(p_{\alpha}(x)\bigr)
  &\ge
  \alpha\,p_{1}(x)\log\bigl(\alpha\,p_{1}(x)\bigr)
 +(1-\alpha)\,p_{2}(x)\log\bigl((1-\alpha)\,p_{2}(x)\bigr) \\[4pt]
  &=\alpha\log\alpha\cdot p_{1}(x)
    +(1-\alpha)\log(1-\alpha)\cdot p_{2}(x)  \\
  &\quad+\alpha\,\phi\bigl(p_{1}(x)\bigr)
        +(1-\alpha)\,\phi\bigl(p_{2}(x)\bigr).
\end{aligned}
\end{equation}
Integrating \eqref{eq:phi-upper} yields
\[
  -{h}(p_{\alpha})
  \;\le\;
  \alpha\bigl(-{h}(p_{1})\bigr)
  +(1-\alpha)\bigl(-{h}(p_{2})\bigr)
  <\infty,
\]
because each component entropy is finite.  Therefore
\({h}(p_{\alpha})>-\infty\).

Likewise, integrating \eqref{eq:phi-lower} gives
\[
  -{h}(p_{\alpha})
  \;\ge\;
  \alpha\log\alpha+(1-\alpha)\log(1-\alpha)
  +\alpha\bigl(-{h}(p_{1})\bigr)
  +(1-\alpha)\bigl(-{h}(p_{2})\bigr).
\]
Hence \(-{h}(p_{\alpha})\) is finite on both sides, so
\({h}(p_{\alpha})\in\mathbb R\) and the Shannon entropy of the mixture is
finite.
\end{proof}

\end{proof}

\section{Auxiliary Results for Convergence Analysis in Pure Synthetic Data Retraining}
\subsection{Auxiliary for Regime (i)}
In this section, we establish convergence in Regime (i): $\alpha=0$, $K<\infty$. For this subsection, we always assume Setting \ref{setting:variable-noise}, Retraining Assumptions \ref{re_ass:1_Shannon-entropy}–\ref{re_ass:3_ess-bdd-Q}, and Assumption \ref{ass1:non0-mass}.

Define the scalar
\[
  R_t \;:=\; \E_{X\sim p_t}\bigl[e^{r(X)}\bigr].
\]
Fix $K\ge2$. Recall that $E_j:=e^{\varepsilon_j}>0$ are i.i.d.\ and independent of all $X_i$'s. Set
\[
  S:=\sum_{i=1}^{K-1} e^{r(X_i)}\,E_i.
\]
Let
\[
  \widetilde{H}_{p_t}^{K}(q)
  \;:=\;
  \E\!\left[\,K\;\frac{q\,E_0}{qE_0+S}\right],\qquad q>0.
\]

\begin{comment}
    \begin{lemma}\label{lem:mon-conc}
For fixed $p_t\in \mathcal P_\pi$, the function $q\mapsto \widetilde{H}_{p_t}^{K}(q)$ is strictly increasing and concave on $(0,\infty)$.
Moreover,
\[
(\widetilde{H}_{p_t}^{K})'(q)=\E\!\left[\frac{K\,E_0\,S}{(qE_0+S)^2}\right]>0,
  \qquad
  (\widetilde{H}_{p_t}^{K})''(q)=\E\!\left[-\frac{2K\,E_0^2\,S}{(qE_0+S)^3}\right]<0.
\]
\end{lemma}
\end{comment}

We will use that $q\mapsto   \widetilde{H}_{p_t}^{K}(q)$ is strictly increasing and concave, $\widetilde{H}_{p_t}^{K}(e^{r(X)})={H}_{p_t}^{K}(X)$,
and that the normalization $\E_{X\sim p_t}[{H}_{p_t}^{K}(X)]=1$ holds.

Let $r_*=\operatorname*{ess\,sup}_{x\in\mathcal{X}} r(x)$. It is obvious that
\[
  A:=\{x\in\mathcal X: Q(x)=Q_*\}=\{x\in\mathcal X: r(x)=r_*\}
  \ \text{and}\ 
  h_t := \sup_{x\in\mathcal X} H_{p_t}^{K}(x) = H_{p_t}^{K}(x')\geq 1,x' \in A.
\]

\begin{lemma}[One-step improvement identity]\label{lem:cov_A1}
For any $t\geq 0$,
\[
  R_{t+1}
  =\E_{p_t}\!\bigl[e^{r(X)}\,H_{p_t}^{K}(X)\bigr],
  \qquad
  R_{t+1}-R_t
  = \operatorname{Cov}_{X\sim p_t}\!\bigl(e^{r(X)},\widetilde{H}_{p_t}^{K}(e^{r(X)}) \bigr)\;\ge 0.
\]
\end{lemma}
\begin{proof}[Proof of Lemma~\ref{lem:cov_A1}]
    By update rule \eqref{equ:update_regime_1}, we know that
\[
 R_{t+1}
  =\E_{p_t}\!\bigl[e^{r(X)}\,H_{p_t}^{K}(X)\bigr].
\]
A direct computation shows that
\begin{align*}
    R_{t+1}-R_t= \mathbb{E}_{X\sim p_t} e^{r(X)} ( \widetilde{H}_{p_t}^{K}(e^{r(X)}) -1) = \operatorname{Cov}_{X\sim p_t}\!\bigl(e^{r(X)},\widetilde{H}_{p_t}^{K}(e^{r(X)}) \bigr).
\end{align*}
    Because $\widetilde{H}_{p_t}^{K}(q)$ is non-decreasing in $q$, we know that $\operatorname{Cov}_{X\sim p_t}\!\bigl(e^{r(X)},\widetilde{H}_{p_t}^{K}(e^{r(X)}) \bigr)\geq 0$.
\end{proof}

\begin{lemma}\label{lem:massA}
For every $t\ge1$ and all $x\in A$,
\begin{equation}
    \frac{p_t(x)}{p_0(x)} \;=\; \prod_{s=0}^{t-1} h_s,\mathbb{P}_t(A) \;=\; \mathbb{P}_0(A)\,\prod_{s=0}^{t-1} h_s,\text{ and }\prod_{s=0}^{t-1} h_s \;\le\; \frac{1}{\mathbb{P}_0(A)}.
\end{equation}
Moreover, $h_t \to 1$ as $t\to\infty$, and $p_t(x)\leq \frac{1}{\mathbb{P}_0(A)} p_0(x)$ for any $x\in \mathcal{X}$.
\end{lemma}

\begin{proof}[Proof of Lemma~\ref{lem:massA}]
The population update is $p_{t+1}(x)=p_t(x)\,H^K_{p_t}(x)$. For $x\in A$,
\[
p_{t+1}(x)=p_t(x)\,h_t
\quad\Longrightarrow\quad
\frac{p_t(x)}{p_0(x)}=\prod_{s=0}^{t-1}h_s,
\]
which proves the first claim. Integrating $p_{t+1}=p_t\,H^K_{p_t}$ over $A$ yields $\mathbb{P}_{t+1}(A)=h_t\,\mathbb{P}_t(A)$, and induction gives $\mathbb{P}_t(A)=\mathbb{P}_0(A)\prod_{s=0}^{t-1}h_s$, the second claim. Since $\mathbb{P}_t(A)\le1$, we obtain $\prod_{s=0}^{t-1}h_s\le 1/\mathbb{P}_0(A)$, the third claim. Notice that $h_t\geq 1$ and $\prod_{s=0}^\infty h_s<\infty$, which implies $h_t\to 1$ as $t\to \infty$, thereby proving the fourth claim. By the definition of $h_t$, we know that
\begin{equation*}
    \frac{p_{t+1}(x)}{p_t(x)}=H^K_{p_t}(x)\leq h_t.
\end{equation*}
Hence, for any $x\in \mathcal{X}$, we have
\begin{equation}\label{ineq:P_t_B}
    p_t(x) \leq \prod_{s=0}^{t-1}h_s\cdot p_0(x) \leq \frac{1}{\mathbb{P}_0(A)}p_0(x).
\end{equation}

\end{proof}

The following two lemmas are straightforward; we omit the proofs.
\begin{lemma}\label{lem:Jensen-G}
Let $a_i:=e^{r(X_i)} /e^{r_*}\in(0,1]$ and $\mu_t:=\E_{p_t}[a_i]=R_t/e^{r_*}$.
For fixed $E=(E_0,\dots,E_{K-1})$ define
\[
  \Phi(a_1,\dots,a_{K-1};E)
  \;:=\;
  \frac{K\,E_0}{E_0+\sum_{i=1}^{K-1} a_i E_i}.
\]
Then $\Phi$ is jointly convex and coordinatewise decreasing in $(a_1,\dots,a_{K-1})$, and
\[
  h_t
  \;= \; \E_E\,\E_{X}\!\bigl[\Phi(a_1,\dots,a_{K-1};E)\mid E\bigr]
  \;\ge\; \E_E\!\bigl[\Phi(\mu_t,\dots,\mu_t;E)\bigr]
  \;=:\; G(\mu_t),
\]
where
\[
  G(\mu)
  \;:=\;
  \E\!\left[\frac{K\,E_0}{E_0+\mu\sum_{i=1}^{K-1}E_i}\right],\qquad \mu\in(0,1].
\]
\end{lemma}

\begin{lemma}[Strict decrease and normalization of $G$]\label{lem:G_A4}
The map $G:(0,1]\to\mathbb R$ is continuous, strictly decreasing, and satisfies $G(1)=1$.
By exchangeability,
\[
  G(1)
  \;=\; \E\!\left[\frac{K\,E_0}{E_0+\sum_{i=1}^{K-1}E_i}\right]
  \;=\; 1.
\]
\end{lemma}

\begin{proof}[Proof of Lemma~\ref{lem:Ri_Rt-to-Qstar}]
Under Setting \ref{setting:variable-noise}, it suffices to show that
\begin{equation}
    R_{t+1}\geq R_t,
\end{equation}
and
\begin{equation}
      \qquad \lim_{t\to\infty} R_t = e^{r_*}.
\end{equation}
By Lemma~\ref{lem:cov_A1}, $R_t$ is nondecreasing and bounded above by $e^{r_*}$.
By Lemma~\ref{lem:massA}, $h_t\to1$.
By Lemma~\ref{lem:Jensen-G}, $h_t\ge G(R_t/e^{r_*})$; by Lemma~\ref{lem:G_A4}, $G$ is strictly decreasing with $G(1)=1$.
Thus $G(R_t/e^{r_*})\le h_t\to1=G(1)$ forces $R_t/e^{r_*}\to1$, i.e.\ $R_t\to e^{r_*}$ as $t\to\infty$.

\end{proof}

\subsection{Auxiliary for Regime (ii)}
In this section, we establish convergence in Regime (ii): $\alpha=0$, $K=\infty$. For this subsection, we always assume Setting \ref{setting:process-noise}, Retraining Assumptions \ref{re_ass:1_Shannon-entropy}–\ref{re_ass:3_ess-bdd-Q}, and Assumption \ref{ass1:non0-mass}.

\begin{lemma}\label{lem:Var_Z_lower}
Let $Z\ge 0$ be a random variable such that $\mathbb{P}(Z=0)\ge \delta$ for some $\delta\in(0,1)$ and $\mathbb{E}[Z]\ge \eta$ for some $\eta>0$. Then
\[
\operatorname{Var}(Z)\ \ge\ \frac{\delta}{1-\delta}\,\eta^{2}.
\]
\end{lemma}

\begin{proof}[Proof of Lemma~\ref{lem:Var_Z_lower}]
Write $p:=\mathbb{P}(Z=0)$ and $q:=1-p=\mathbb{P}(Z>0)$. By Cauchy--Schwarz applied to $Z$ and $\mathbf{1}_{\{Z>0\}}$,
\[
\mathbb{E}[Z]
= \mathbb{E}\!\left[ Z\,\mathbf{1}_{\{Z>0\}} \right]
\le \bigl(\mathbb{E}[Z^{2}]\bigr)^{1/2}\,\bigl(\mathbb{P}(Z>0)\bigr)^{1/2}
= \bigl(\mathbb{E}[Z^{2}]\bigr)^{1/2} q^{1/2}.
\]
Hence $\mathbb{E}[Z^{2}] \ge \mathbb{E}[Z]^2/q$. Therefore
\[
\operatorname{Var}(Z)
= \mathbb{E}[Z^{2}] - \mathbb{E}[Z]^2
\ge \mathbb{E}[Z]^2\!\left(\frac{1}{q}-1\right)
= \frac{p}{q}\,\mathbb{E}[Z]^2.
\]
Using $p\ge \delta$ and $\mathbb{E}[Z]\ge \eta$, together with the fact that $x\mapsto x/(1-x)$ is increasing on $(0,1)$, we obtain \(\operatorname{Var}(Z)\ \ge\ \frac{\delta}{1-\delta}\,\eta^{2},\) as claimed.
\end{proof}

\begin{proof}[Proof of Lemma~\ref{lem:Rii_Rt-to-Qstar}]
Set $C_t=\mathbb{E}_{X\sim p_{t}}[Q(X)]$. For Regime (ii), we have
\[
\mathbb{E}_{X\sim p_{t+1}}Q(X)=\mathbb{E}_{X\sim p_{t}}Q(X)H_{p_t}^\infty(X)=\mathbb{E}_{X\sim p_{t}}Q^2(X)/C_t. 
\]
Thus,
\begin{equation*}
  { C_{t+1}-C_{t} }
    = \frac{\mathbb{E}_{X\sim p_t}[Q^2(X)] - \bigl(\mathbb{E}_{X\sim p_t}[Q(X)]\bigr)^2}{\mathbb{E}_{X\sim p_t}[Q(X)]}
    = \frac{\operatorname{Var}_{X\sim p_t}(Q(X))}{C_t}
    \;\geq\; 0,
\end{equation*}
which implies that $C_t\geq C_0>0$. Now, we obtain
\begin{equation*}
    C_{t+1}-C_t \;\geq\; \frac{\operatorname{Var}_{X\sim p_t}(Q(X))}{Q^*} \;\geq\; 0.
\end{equation*}
Notice that
\begin{equation}
   \sum_{i=0}^{t} \operatorname{Var}_{X\sim p_i}(Q(X))
   \;\leq\; Q_* \bigl(C_{t+1} - C_0\bigr)
   \;\leq\; Q_* \bigl(Q_* - C_0\bigr)
   \;<\; \infty.
\end{equation}
It follows that
\begin{equation}\label{equ:Var_p_Q_0}
    \lim_{t\to\infty} \operatorname{Var}_{X\sim p_t}(Q(X)) \;=\; 0.
\end{equation}
If there exists $\eta>0$ such that, for any $t\geq 1$,
\begin{equation}
    C_t \;\leq\; Q_* - \eta,
\end{equation}
then applying Lemma~\ref{lem:Var_Z_lower} to $Q_* - Q(X) \geq \eta > 0$ (for any $t\geq 1$) yields
\begin{equation}
    \operatorname{Var}_{X\sim p_t}(Q(X)) \;\geq\; \frac{\delta}{1-\delta}\,\eta^2 \;>\; 0,
\end{equation}
which contradicts \eqref{equ:Var_p_Q_0}. Hence, we obtain that 
\[
\limsup_{t\to\infty} C_t = Q_*.
\]
Since $C_t$ is a nondecreasing sequence, we have
\[
\lim_{t\to\infty} C_t = Q_*.
\]
\end{proof}

\subsection*{Convergence Analysis in pure synthetic data retraining}
%We will show the convergence result under Regime (i) and (ii) in pure synthetic data retraining ($\alpha=0$). 

We show convergence under Regimes~(i) and~(ii) in the purely synthetic data retraining setting ($\alpha=0$).

\begin{lemma}\label{lem:A5_p_t}
    If $\mathbb{E}_{X\sim p_t}Q(X) \to Q_*$ as $t\to \infty$, then 
    \[
    \lim_{t\to \infty} \mathbb{P}_t(A^c) =0.
    \]  
\end{lemma}
\begin{proof}[Proof of Lemma~\ref{lem:A5_p_t}]
Let $B_\eta= \{x\in \mathcal{X} : Q(x) \leq Q_*-\eta\}$. By Markov inequality, for any $\eta>0$,
\begin{equation*}
    \mathbb{P}_t(B_\eta)\leq \frac{Q_*-\E_{X\sim p_t} Q(X) }{ \eta} \to 0,
\end{equation*}
as $t\to \infty$.

Observe that Lemma~\ref{lem:massA} remains valid under Regime~(ii). Hence, for any $B\subset \mathcal{X}$, by Lemma~\ref{lem:massA}, we have
\begin{equation}\label{ineq:P_t_BB}
    \mathbb{P}_t(B) \leq \prod_{s=0}^{t-1}h_s\cdot \mathbb{P}_0(B) \leq \frac{\mathbb{P}_0(B)}{\mathbb{P}_0(A)}.
\end{equation}
Applying \eqref{ineq:P_t_BB}, we have that for any $\eta>0$
\begin{align*}
    \limsup_{t\to \infty} \mathbb{P}_t(A^c)&\leq \limsup_{t\to \infty} \mathbb{P}_t(B_\eta) + \limsup_{t\to \infty} \mathbb{P}_t(A^c\backslash B_\eta)\\
    &\leq \frac{1}{\mathbb{P}_0(A) }  \mathbb{P}_0(A^c\backslash B_\eta).
\end{align*}
Letting $\eta \to 0$, the sets $A^c\backslash B_\eta$ converge to the empty set. By the continuity of probability measures, we have
\[
\lim_{\eta\to 0^+}\mathbb{P}_0(A^c\backslash B_\eta) = \mathbb{P}_0(\emptyset) =0.
\]
In conclusion, we obtain
\[
 \limsup_{t\to \infty} \mathbb{P}_t(A^c)=0.
\]
\end{proof}

\begin{proof}[Proof of Theorem~\ref{thm:KL-uniform-RegimeI}]
By the update rule \(p_{t+1}=p_t H_{p_t}^K\) (Theorem~\ref{thm:pure-Kfinite}) for Regime (i), the choice kernel \(H_{p_t}^K(x)\) depends on \(x\) only through \(e^{r(x)}\) (hence through \(Q(x)\)); in particular, it is constant on the level set \(A\). By the update rule \(p_{t+1}=p_t H_{p_t}^\infty\) (Theorem~\ref{thm:pure-Kinfty}) for Regime (ii), the choice kernel \(H_{p_t}^\infty(x)\) depends on \(x\) only through \(Q(x)\); in particular, it is constant on the level set \(A\).

Thus, for both Regime (i) and (ii), writing \(h_t:=H_{p_t}^K(x)\) for any \(x\in A\), Lemma~\ref{lem:massA} yields
\[
p_t(x)=\Bigl(\prod_{s=0}^{t-1}h_s\Bigr)\,p_0(x)\quad\text{for all }x\in A
\quad\text{and}\quad 
\mathbb P_t(A)=\mathbb P_0(A)\prod_{s=0}^{t-1}h_s.
\]
Therefore, on \(A\) we have the exact identity
\begin{equation}\label{eq:pt-on-A}
    p_t(x)=\mathbb P_t(A)\,\frac{p_0(x)}{\mathbb P_0(A)}=\mathbb P_t(A)\,p_*(x).
\end{equation}
It follows immediately that
\(
\sup_{x\in A}\bigl|\frac{p_t(x)}{p_*(x)}-1\bigr|
=|\mathbb P_t(A)-1|.
\)

By Lemma~\ref{lem:Ri_Rt-to-Qstar} we have \(\E_{p_t}[Q(X)]\to Q_*\) in Regime~(i), and thus Lemma~\ref{lem:A5_p_t} gives \(\mathbb P_t(A^c)\to0\), i.e.\ \(\mathbb P_t(A)\to1\) as $t\to \infty$.

Using \eqref{eq:pt-on-A} and \(p_*(A)=1\),
\[
\KL(p_*\,\|\,p_t)
=\int_A p_*(x)\log\!\frac{p_*(x)}{p_t(x)}\,\pi(dx)
=\int_A p_*(x)\log\!\frac{1}{\mathbb P_t(A)}\,\pi(dx)
=-\log \mathbb P_t(A).
\]
Since \(\mathbb P_t(A)\to1\), we conclude \(\KL(p_*\,\|\,p_t)\to0\) as $t\to \infty$.
\end{proof}

\subsection{Auxiliary for Regime (iii)}

\begin{proof}[Proof of Lemma~\ref{lem:contraction}]
Let \(h_B(x_{1:K},\varepsilon_{1:K})=\sum_{j=1}^K \tilde{p}_j(x_{1:K},\varepsilon_{1:K})\,\mathbf{1}_B(x_j)\). Let $F$ denote the joint law of $\varepsilon_{1:K}$. 
In what follows, we use the notation \(f^{\otimes N}\) to denote the $N$-fold product measure associated with a law having density $f$. By a slight abuse of notation, we also write $f^{\otimes N}$ for the $N$-fold product measure generated by $f$ when $f$ denotes a probability measure or a law of a random variable.

Applying Corollary~\ref{cor:selection-identity}, for any measurable set $B$, we have
\[
|\mathbb{S}_w(B)-\mathbb{S}_u(B)|
=\Big|\mathbb E_{w^{\otimes K}\otimes F}h_B-\mathbb E_{u^{\otimes K}\otimes F}h_B\Big|=2\Big|\mathbb E_{w^{\otimes K}\otimes F}  (h_B-1)/2-\mathbb E_{u^{\otimes K}\otimes F}(h_B-1)/2 \Big|.
\]
Since \(  |(h_B-1)/2|\le 1\), by Theorem 7.7 in \cite{polyanskiy2025information}, we have
\[
\Big|\mathbb E_{w^{\otimes K}\otimes F}  (h_B-1)/2-\mathbb E_{u^{\otimes K}\otimes F}(h_B-1)/2 \Big|
\le \frac{1}{2} d_{\rm TV}\!\big(w^{\otimes K}\otimes F,\,u^{\otimes K}\otimes F\big).
\]
Taking the supremum over \(B\) gives
\[
d_{\rm TV}\!\big( \mathbb{S}_w, \mathbb{S}_u \big) = \sup_B |\mathbb{S}_w(B)-\mathbb{S}_u(B)| \leq  d_{\rm TV}\!\big(w^{\otimes K}\otimes F,\,u^{\otimes K}\otimes F\big).
\]
By Proposition 7.2 in \cite{polyanskiy2025information}, we have
\[
d_{\rm TV}\!\big(w^{\otimes K}\otimes F,\,u^{\otimes K}\otimes F\big)
=d_{\rm TV}\!\big(w^{\otimes K},u^{\otimes K}\big).
\]
By the triangle inequality and a telescoping argument,
\[
d_{\rm TV}\big(w^{\otimes K},u^{\otimes K}\big)\;\le\; K\cdot d_{\rm TV}(w,u).
\]
The bound \eqref{equ:T_contraction} for \(T\) is immediate.
\end{proof}

\begin{proof}[Proof of Theorem~\ref{thm:tv-convergence-iii}]
By Lemma~\ref{lem:contraction}, $d_{\rm TV}(\mathbb{S}_w,\mathbb{S}_u)\le K\,d_{\rm TV}(w,u)$, hence
$d_{\rm TV}(Tw,Tu)\le (1-\alpha)K\cdot d_{\rm TV}(w,u)=\rho\,d_{\rm TV}(w,u)$.
Since $(\mathcal{P}_\pi,d_{\rm TV})$ is complete (it embeds isometrically into $L^1(\pi)$), Banach’s fixed‑point theorem yields
existence, uniqueness, and the stated geometric rate.
\end{proof}

\begin{proof}[Proof of Lemma~\ref{lem:r_iii_reward_increase}]
Let $R_t=\mathbb{E}_{X\sim p_t}Q(X)$. We have
\begin{equation*}
    R_{t+1}=\alpha R_0+ (1-\alpha) \mathbb{E}_{X\sim p_t}Q(X) H^K_{p_t}(X) .
\end{equation*}
Setting $D_t=R_t-R_0$, we have
\begin{equation}\label{equ:48_D_t}
    D_{t+1}=(1-\alpha)\left( \mathbb{E}_{X\sim p_t}Q(X) H^K_{p_t}(X)-\mathbb{E}_{X\sim p_t}Q(X) +D_t  \right).
\end{equation}
Similar to the proof of Lemma~\ref{lem:cov_A1}, we have
\begin{equation*}
    \mathbb{E}_{X\sim p_t}e^{r(X)}H^K_{p_t}(X) - \mathbb{E}_{X\sim p_t}e^{r(X)}  =\operatorname{Cov}_{p_t}\Big(e^{r(X)} , H_{p_t}^K(X)\Big)\geq 0,
\end{equation*}
which implies that
\begin{equation*}
    \mathbb{E}_{X\sim p_t} Q(X)H^K_{p_t}(X) - \mathbb{E}_{X\sim p_t}Q(X) =\operatorname{Cov}_{p_t}\Big(Q(X) , H_{p_t}^K(X)\Big)\geq 0.
\end{equation*}
Hence, \eqref{equ:48_D_t} implies that
$D_{t+1}\geq (1-\alpha)D_{t}\geq (1-\alpha)^tD_1$. Thus, for any positive integer $t$,
\begin{equation*}
    R_t\geq R_0+ (1-\alpha)^t  D_1.
\end{equation*}
As in the proof of Lemma~\ref{lem:cov_A1}, we have
$$ 
\operatorname{Cov}_{p_0}(Q(X),H_{p_0}^K(X))=\operatorname{Cov}_{ p_0}\!\bigl(e^{r(X)},\widetilde{H}_{p_0}^{K}(e^{r(X)}) \bigr)
$$
where $\widetilde{H}_{p_0}^{K}(y)$ is non‑decreasing in $y$. Since $Q_*$ is the essential supremum of $Q$ and $\mathbb{E}_{X\sim p_0}Q(X)\in (0,Q_*)$, the random variable $\widetilde{H}_{p_0}^{K}(e^{r(X)})$ is not almost surely constant, hence $e^{r(X)}$ is also not almost surely constant. Using 
\begin{equation*}
\operatorname{Cov}_{p_0}\!\bigl(e^{r(X)},\widetilde{H}_{p_0}^{K}(e^{r(X)}) \bigr)
= \frac{1}{2}\,
\mathbb{E}_{X,X'\,\overset{\mathrm{i.i.d.}}{\sim}\,p_0}
\Bigl[
\bigl(e^{r(X)}-e^{r(X')}\bigr)
\bigl(\widetilde{H}_{p_0}^{K}(e^{r(X)})-\widetilde{H}_{p_0}^{K}(e^{r(X')})\bigr)
\Bigr],
\end{equation*}
and the integrand is almost surely nonnegative, with strict positivity on a set of positive probability because both arguments are non‑decreasing and not almost surely constant. Therefore $\operatorname{Cov}_{p_0}(Q(X),H_{p_0}^K(X))>0$. Combined with $D_1=(1-\alpha)\operatorname{Cov}_{p_0}(Q(X),H_{p_0}^K(X)) >0$, we obtain \eqref{ineq:49_need}.

Similarly, we have
\begin{align*}
    \lim_{t\to \infty} \mathbb{E}_{X\sim p_t} Q(X)&=  \mathbb{E}_{X\sim p_0} Q(X) +(1-\alpha)\lim_{t\to \infty} \mathbb{E}_{p_t}[Q(X)H_{p_t}^K(X)]\\
    &=  \mathbb{E}_{X\sim p_0} Q(X) +(1-\alpha)  \operatorname{Cov}_{p_*}(Q(X),H_{p_*}^K(X)) \\
    &>\mathbb{E}_{X\sim p_0} Q(X).
\end{align*}

\end{proof}

\subsection{Auxiliary for Regime (iv)}

\begin{proof}[Proof of Lemma~\ref{lem:dynamics_w}]
Rewriting update \eqref{equ:update_Regime_iii_iv} in Regime (iv), we have
\begin{equation}
    p_{t+1}(x)=\alpha p_{\rm ref}(x)+(1-\alpha)\frac{Q(x)p_t(x)}{\int_\mathcal{X} Q(x)p_t(x)\pi(dx)},
\end{equation}
which can be rewritten as 
\[
w_{t+1}(x) p_{\rm ref}(x) = \alpha p_{\rm ref}(x) +(1-\alpha) \frac{Q(x)w_{t}(x) p_{\rm ref}(x) }{ \innerproduct{w_t}{Q}_{\rm ref}}
\]
Hence, we obtain \eqref{equ:update_w_Regime_iv}. Notice that \eqref{equ:update_w_Regime_iv} implies that
\[
w_{t+1}\propto L[w_t].
\]
Since $L$ is linear, we have
\[
w_t(x)\propto L[w_t]\propto L^2[w_{t-2}](x)  \cdots \propto L^t[w_{0}](x).
\]
There exists constant $C$ such that $w_t(x)=C\cdot L^t[w_{0}](x)$. Since $\innerproduct{w_t}{1}_{\rm ref}=\int_{\mathcal X} w_t(x) \mathbb{P}_{\rm ref}(dx)=1$, we obtain \eqref{equ:update_w_Regime_iv}.
\end{proof}

\begin{proof}[Proof of Lemma~\ref{lem:fixed_point_c}]
For \(c\in (Q_{0},Q_{*}]\) with \(Q_{0}:=(1-\alpha)Q_{*}\), define
\[
  h(x;c)
  \;:=\;
  \frac{\alpha}{\,1-(1-\alpha)Q(x)/c\,}.
\]

\begin{itemize}
\item \emph{Upper bound.}  At \(c=Q_{*}\), \(h(x;Q_*)\leq 1\) pointwise, and hence
      \(\innerproduct{h(\cdot;Q_{*})}{1}_{\mathrm{ref}}\le 1\).

\item \emph{Lower-limit behavior.}
      Assumption~\ref{ass:A3} implies
      \[
        \lim_{c\downarrow Q_{0}}
        \innerproduct{h(\cdot;c)}{1}_{\mathrm{ref}}
        \;>1.
      \]

\item \emph{Monotonicity.}
      The map
      \(c\mapsto\innerproduct{h(\cdot;c)}{1}_{\mathrm{ref}}\)
      is strictly decreasing on \((Q_{0},Q_{*}]\).
\end{itemize}
By the intermediate value theorem, there exists a unique \(c_{*}\in(Q_{0},Q_{*}]\) such that  
\(\innerproduct{h(\cdot;c_{*})}{1}_{\mathrm{ref}}=1\).

\noindent Define \(w_{*}(x):=h(x;c_{*})\). Observe that
\begin{align*}
    \int_{\mathcal{X}} \frac{\alpha }{ 1- (1-\alpha)Q(x)/c_*}\mathbb{P}_{\rm ref}(x)=1 \quad \text{and} \quad \int_{\mathcal{X}} 1\,\mathbb{P}_{\rm ref}(x)=1,
\end{align*}
which implies that
\begin{align*}
    \innerproduct{w_*}{Q}_{\rm ref} &= \int_{\mathcal{X}} \frac{\alpha Q(x) }{ 1- (1-\alpha)Q(x)/c_*}\mathbb{P}_{\rm ref}(x)\\
    &=\frac{1}{1-\alpha}c_*\int_{\mathcal{X}} \frac{\alpha }{ 1- (1-\alpha)Q(x)/c_*}\mathbb{P}_{\rm ref}(x)-\frac{\alpha}{1-\alpha}c_* \int_{\mathcal{X}} 1\,\mathbb{P}_{\rm ref}(x)=c_*.
\end{align*}
A direct calculation shows that
\[
  L_{ N}[w_{*}](x)=w_{*}(x),
  \qquad
  \innerproduct{w_{*}}{1}_{\mathrm{ref}}=1,
\]
and therefore \(w_{*}\) is the fixed point of the nonlinear operator $L_N$ over $\mathcal{P}_{\mathbb{P}_{\rm ref}}$.
\end{proof}

\begin{lemma}\label{lem:w=w_*}
    If a probability density $w\in \mathcal{P}_{\mathbb{P}_{\rm ref}}$ satisfies $L_N^k[w]=w_*$ a.s.\ for some $k\geq 1$, then $w=w_*$ a.s.
\end{lemma}

\begin{proof}[Proof of Lemma~\ref{lem:w=w_*}]
Let $v=L_N^{k-1}[w]$, which is also a probability density. Then $L_N[v]=w_*$. Since $w_*$ is a fixed point of $L_N$, we have
\[
    \alpha+(1-\alpha)\frac{v(x)Q(x)}{\innerproduct{v}{Q}_{\rm ref}}
    \;=\;
    \alpha+(1-\alpha)\frac{w_*(x)Q(x)}{\innerproduct{w_*}{Q}_{\rm ref}}.
\]
This equality implies that
\[
    v(x) \;\propto\; w_*(x)\quad \text{a.s.}
\]
Because both $v$ and $w_*$ are probability densities, the proportionality constant must be one; hence $v=w_*$ a.s.  
Therefore $L_N^{k-1}[w]=w_*$ a.s. Repeating this argument iteratively, we conclude $w=w_*$ a.s.
\end{proof}

\begin{proof}[Proof of Theorem~\ref{thm:Regime_iv_convergence}]
By Assumption \ref{ass:A2_new}, there exists $M>0$ such that \(w_{0}(x)\le M<\infty\) $\mathbb{P}_{\rm ref}$‑a.s. Thus,
\[
  \alpha \leq w_{1}(x)=L_{ N}[w_{0}](x)=\alpha+(1-\alpha) \frac{w_{0}(x)\,Q(x)}{\innerproduct{w_{0}}{Q}_{\mathrm{ref}}}\leq\alpha+\frac{(1-\alpha)M Q_{*}}{ Q_{\min} }=:M_{1}<\infty .
\]
Using the triangle inequality of the Hilbert projective metric,
\[
  d_\mathcal{H}(w_{1},w_{*})
  \;\le\;
  d_\mathcal{H}(w_{1},\mathbf 1)
  +d_\mathcal{H}(\mathbf 1,w_{*})
  \;\le\;
  \log\bigl( \frac{M_1}{\alpha}\bigr)
  +\log\Bigl(\frac1{1-(1-\alpha)Q_{*}/c_{*}}\Bigr)
  \;<\;\infty.
\]
By Birkhoff’s contraction theorem \citep{eveson1995applications,eveson1995elementary}
\begin{equation}\label{equ:45_Birkhoff}
      d_\mathcal{H}(w_{k+1},w_{*})
  \;=\;
  d_\mathcal{H} \bigl(L_{ N}[w_{k}],L_{ N}[w_{*}]\bigr)
  \;\le\;
  d_\mathcal{H}(w_{k},w_{*}),
  \qquad k\ge1,
\end{equation}
which inductively yields the asserted monotone chain
\(
  d_\mathcal{H}(w_{k+1},w_{*})
  \le\ldots\le
  d_\mathcal{H}(w_{1},w_{*})<\infty.
\) Set 
\begin{align} \label{R}
    R= d_{\mathcal{H}}(w_{1},w_*)<\infty.
\end{align}

Recall that $\beta(w,w_*)=\operatorname*{ess\,sup}_{x}\frac{w(x)}{w_*(x)}$ and $w_k=L_N^k [w_0]$, where
\[
  L_{ N}[w](x)
  \;=\;
  \alpha
  \;+\;
  (1-\alpha)\,
  \frac{w(x)\,Q(x)}
       {\innerproduct{w}{Q}_{\mathrm{ref}}}
  \;=\;  \frac{L[w](x)} {\innerproduct{w}{Q}_{\mathrm{ref}}}.
\]
Set $M_k= \beta(w_k,w_*)$ and $m_k= \beta(w_*,w_k)$, where $w_k$ and $w_*$ are both PDFs. By Lemma~\ref{lem:w=w_*}, $d_{\mathcal{H}}(w_0,w_*)>0$ implies that $w_k$ and $w_*$ are linearly independent. Thus, $M_k>1$ and $m_k>1$. Furthermore, we know that 
\begin{equation}\label{equ:some_bound_31}
    \innerproduct{M_k w_*-w_k}{Q}_{\rm ref}\geq Q_{\rm min} (M_k-1) \text{ and }\innerproduct{m_k w_k-w_*}{Q}_{\rm ref}\geq Q_{\rm min} (m_k-1).
\end{equation}
To see why \eqref{equ:some_bound_31} is true, write
\begin{align*}
    \innerproduct{M_k w_*-w_k}{Q}_{\rm ref} &= \int \left( M_kw_*(x)-w_k(x) \right)Q(x) \cdot p_{\rm ref}(x) d\pi(x) \\
    &\geq Q_{\rm min} \cdot \left[ M_k \int w_*(x) \cdot p_{\rm ref}(x) d\pi(x) - \int w_k(x) \cdot p_{\rm ref}(x) d\pi(x)  \right] \\
    &= Q_{\rm min} \cdot (M_k-1).
\end{align*}
The second inequality in \eqref{equ:some_bound_31} holds by using similar arguments.

% By definition, $d_{\mathcal{H}}(w_k,w_*)=\log(M_km_k)$. Hence, 
% \begin{equation*}
%     d_{\mathcal{H}}(w_{k+1},w_*)=d_{\mathcal{H}}(Lw_{k}/\innerproduct{w_k}{Q}_{\mathrm{ref}},w_*)=d_{\mathcal{H}}(Lw_{k} ,w_*)=\log(\beta(Lw_{k} ,w_*)\beta(w_*,Lw_{k} )).
% \end{equation*}
Recall $R$ in \eqref{R}. Since $d_{\mathcal{H}}(w_k,w_*)\leq d_{\mathcal{H}}(w_1,w_*)\leq R$, for all $k\geq 1$, we have 
\[
\operatorname*{ess\,sup}_{x} \frac{w_k(x)}{w_*(x)} \leq e^{R} \cdot \operatorname*{ess\,inf}_{x} \frac{w_*(x)}{w_k(x)} \leq e^R.
\]
This further implies
\[
\operatorname*{ess\,sup}_{x} w_k(x) \leq e^R w_{*,\rm{max}}
\]
where $w_{*,\rm{max}}=\sup_x w_*(x)<\infty$.

Applying \eqref{equ:some_bound_31}, a direct computation shows that
\begin{align*}
    &\beta(Lw_{k} ,Lw_*) = \operatorname*{ess\,sup}_{x}\frac{Lw_k(x)}{Lw_*(x)}\\
    =&\operatorname*{ess\,sup}_{x} \frac{\alpha \innerproduct{w_k}{Q}_{\rm ref}+ (1-\alpha) w_k(x)Q(x) }{\alpha \innerproduct{w_*}{Q}_{\rm ref}+ (1-\alpha) w_*(x)Q(x)}\\
    \leq&M_k\operatorname*{ess\,sup}_{x} \frac{\alpha \innerproduct{w_*}{Q}_{\rm ref}+ (1-\alpha) w_*(x)Q(x) - \alpha Q_{\rm{min}} (1-1/M_k) }{\alpha \innerproduct{w_*}{Q}_{\rm ref}+ (1-\alpha) w_*(x)Q(x)}\\
    \leq&M_k \frac{(\alpha + (1-\alpha) e^R w_{*,\rm{max}}  )Q_* -\alpha Q_{\rm min} (1-1/M_k) }{(\alpha + (1-\alpha) e^R w_{*,\rm{max}} )Q_*}\\
    =&M_k g(M_k),
\end{align*}
where the second inequality in the display above follows from the fact that the function $x \mapsto (x-c)/x$ is increasing on $\mathbb{R}^+$ for all $c>0$, and 
\[
g(y) :=\frac{(\alpha + (1-\alpha) e^R w_{*,\rm{max}}  )Q_* -\alpha Q_{\rm min} (1-1/y) }{(\alpha + (1-\alpha) e^R w_{*,\rm{max}} )Q_*},\ y\geq 1.
\]
Notice that $g(y)$ is strictly decreasing in $y$, and $g(1)=1$. Thus, $g(y) \in (0,1]$ for any $y\geq 1$. Similarly, we can also show that \(\beta(Lw_*,Lw_{k} ) \leq m_k g(m_k).\) Notice that
\[
d_{\mathcal{H}}( w_{k+1} ,w_*)=d_{\mathcal{H}}(Lw_{k} ,Lw_*) =\log (\beta(Lw_{k} ,Lw_*) \beta(Lw_*, Lw_{k} ) ).
\]
Now, we obtain that
\begin{equation}
    d_{\mathcal{H}}(w_{k+1},w_*) \leq  \log(M_km_k g(M_k)g(m_k)) = d_{\mathcal{H}}(w_{k},w_*) + \log(g(M_k)g(m_k))<d_{\mathcal{H}}(w_{k},w_*).
\end{equation}
Since
\begin{equation}
    0\leq \sum_{i=1}^{k-1} -\log (g(M_i)g(m_i)) \leq  d_{\mathcal{H}}(w_{1},w_*)- d_{\mathcal{H}}(w_{k+1},w_*) \leq d_{\mathcal{H}}(w_{1},w_*)<
    \infty,
\end{equation}
we have
\[
\lim_{k\to \infty }\log (g(M_k)g(m_k)) =0.
\]
Hence, we obtain that both $g(M_k)\to 1$ and $g(m_k)\to 1$ as $k\to \infty$, i.e., $M_k\to 1$ and $m_k\to 1$. It follows that
$d_{\mathcal{H}}(w_k,w_*)=\log(M_km_k) \to 0 $ as $k\to \infty$.
\end{proof}

\begin{proof}[Proof of Lemma~\ref{lem:r_iv_reward_lbound}]
Recall that \(p_t=w_t p_{\rm ref}\), and let \(R_t:=\mathbb{E}_{X\sim p_t}[Q(X)].\) Since \(p_0=p_{\rm ref}\), update \eqref{equ:update_Regime_iii_iv} implies
\[
R_{t+1}
=
\alpha R_0
+
(1-\alpha)
\frac{\mathbb{E}_{X\sim p_t}[Q^2(X)]}
     {\mathbb{E}_{X\sim p_t}[Q(X)]}.
\]
Define \(D_t:=R_t-R_0.\) We obtain
\begin{equation}
D_{t+1}
=
(1-\alpha)
\left(
D_t+
\frac{\operatorname{Var}_{p_t}(Q)}{R_t}
\right).
\label{equ:D_t_update}
\end{equation}

We next derive a uniform lower bound for
\(\operatorname{Var}_{p_t}(Q)\). For every \(t\geq 1\), the update gives
the mixture representation \(p_t
=
\alpha p_0+(1-\alpha)\widetilde p_t.\) Let \(Z_t\) be a Bernoulli mixture indicator satisfying \(\mathbb{P}(Z_t=0)=\alpha,
\mathbb{P}(Z_t=1)=1-\alpha.\)
Conditionally on \(Z_t\), let \(X_t\mid\{Z_t=0\}\sim p_0,
X_t\mid\{Z_t=1\}\sim\widetilde p_t.\)
Then \(X_t\sim p_t\). By the law of total variance, $\operatorname{Var}_{p_t}(Q)
\geq
\mathbb{E}\!\left[
\operatorname{Var}\!\left(Q(X_t)\mid Z_t\right)
\right] =
\alpha\operatorname{Var}_{p_0}(Q(X))$. Consequently,
\( 
\operatorname{Var}_{p_t}(Q)
\geq
\alpha\operatorname{Var}_{p_0}(Q).\) Set \(V_0:=\operatorname{Var}_{p_0}(Q).\)
Since \(R_0\in(0,Q^*)\) and \(Q^*\) is the essential supremum of \(Q\),
the random variable \(Q\) is not \(p_0\)-almost surely constant. Therefore, \(V_0>0.\) Moreover, \(R_t\leq Q^*\). Thus, we have \(D_{t+1}
\geq
(1-\alpha)D_t
+
\frac{\alpha(1-\alpha)V_0}{Q^*}.\)
Since \(D_0=0\), iteration gives $D_t\geq
\frac{\alpha(1-\alpha)V_0}{Q^*}
\sum_{j=0}^{t-1}(1-\alpha)^j =
\frac{(1-\alpha)V_0}{Q^*}
\left(1-(1-\alpha)^t\right).$ Therefore, for every \(t\geq 1\),
\[
R_t-R_0
=
D_t
\geq
\frac{(1-\alpha)V_0}{Q^*}
\left(1-(1-\alpha)^t\right)
>0,
\]
which proves the result.
\end{proof}

\section{Stability Analysis}

\subsection{Regimes (i) and (ii): Unstable under bounded reward perturbations}

\begin{proof}[Proof of Theorem~\ref{thm:TV_instability}]
Set
\[
E=\bigl\{x\in \mathcal X:Q_*-\delta\le Q(x)<Q_*\bigr\},
\quad\text{so that }\,\mathbb{P}_0(E)>0.
\]
Define the perturbation
\[
\Delta r(x)\;:=\;
\begin{cases}
\log\bigl(Q_*/Q(x)\bigr), & x\in E,\\[2pt]
-\eta, & x\in A,\\[2pt]
0, & \text{otherwise.}
\end{cases}
\]
For $x\in E$, the side condition $(Q_*-\delta)e^{\eta}\ge Q_*$ implies $Q(x)\ge Q_*e^{-\eta}$, hence
\[
0\;\le\;\Delta r(x)=\log\bigl(Q_*/Q(x)\bigr)\;\le\;\log\bigl(Q_*/(Q_*e^{-\eta})\bigr)=\eta.
\]
On $A$ we set $\Delta r=-\eta$, and elsewhere $0$, so $\|\Delta r\|_{\infty}=\eta$.

For $x\in E$,
\[
Q_{\Delta r}(x)=Q(x)e^{\Delta r(x)}=Q(x)\,\frac{Q_*}{Q(x)}=Q_*,
\]
while for $x\in A$, $Q_{\Delta r}(x)=Q_*e^{-\eta}<Q_*$. For $x\notin E\cup A$ we have $Q_{\Delta r}(x)=Q(x)\le Q_*-\delta<Q_*$. Hence
\[
\operatorname*{ess\,sup}_y Q_{\Delta r}(y)=Q_*,\qquad A_{\Delta r} = E,\qquad \mathbb{P}_0(A_{\Delta r}) = \mathbb{P}_0(E)>0,
\]
and $A_{\Delta r}\cap A=\varnothing$.

In the fully synthetic $K$-choice loop (both for finite $K$ and for $K=\infty$), the limit distribution concentrates on the top $Q$-level set reached at initialization: $p_{\infty,0}$ is $p_0$ renormalized to $A$, and $p_{\infty,\Delta r}$ is $p_0$ renormalized to $A_{\Delta r}$. Since $A_{\Delta r}\cap A=\varnothing$ and $\mathbb{P}_0(A_{\Delta r})>0$, we obtain
\[
d_{\mathrm{TV}}\!\left(p_{\infty,\Delta r},\,p_{\infty,0}\right)
\;\ge\;\Bigl| \int_{A_{\Delta r}} p_{\infty,\Delta r}(x)\pi(dx)-\int_{A_{\Delta r}} p_{\infty,0}(x)\pi(dx)\Bigr|
\;=\;|1-0|\;=\;1,
\]
and the TV metric is always at most $1$, hence equality holds. 
\end{proof}

\subsection{Stable Regime (iii)}

\begin{lemma}\label{lem:stability_11_general}
Assume there exists $\rho\in [0,1)$ such that for any $p,q\in \mathcal{P}_\pi$ and $\Delta r\in L^\infty(\mathcal X,\mathcal B,\pi)$,
\begin{equation}
    d_{\rm TV}(T_{\Delta r} p, T_{\Delta r} q ) \leq \rho\cdot d_{\rm TV}(p, q).
\end{equation}
Then
\begin{equation}\label{lim:stability_11_general}
 d_{\rm TV}(T^n_{\Delta r} p_0,T^n_{0} p_0) \leq \frac{\delta(\|\Delta r \|_\infty)}{1-\rho},
\end{equation}
where 
\[
\delta(\eta):= \sup_{\| \Delta r \|_\infty\leq \eta } \sup_{p\in \mathcal{P}_{\pi}} d_{\rm TV}(T_{\Delta r}p , T_{0}p )
\]

\end{lemma}

\begin{proof}[Proof of Lemma~\ref{lem:stability_11_general}]
    Let \(D_n:=d_{\rm TV}(T^n_{\Delta r}p_0,T^n_{0}p_0 )\) and $\eta=\|\Delta r\|_\infty$. Then
\[
D_{n+1} \leq   d_{\rm TV}(T_{\Delta r}T^n_{\Delta r}p , T_{0}T^n_{\Delta r}p )+d_{\rm TV}(T_{0}T^n_{\Delta r}p , T_{0}T^n_{0}p )\leq \delta(\eta)+ \rho D_n.
\]
With $D_0=0$, we obtain that 
\[
D_n\leq \delta(\eta) \frac{1-\rho^n}{1-\rho}\leq \frac{\delta(\eta) }{1-\rho}.
\]
Hence, we complete the proof.

\end{proof}

\begin{lemma}\label{lem:kernel-Lip}
For every density $p$ and $x\in\mathcal X$, let $K\in\mathbb{N}$ be finite and let $\Delta r$ be bounded with $\|\Delta r\|_\infty<\infty$. Then
\[
  \bigl|
     H_{p,\Delta r}^{K}(x)
    -H_{p,0}^{K}(x)
  \bigr|\le \tfrac{K}{2}\,\|\Delta r\|_{\infty}.
\]
\end{lemma}

\begin{proof}[Proof of Lemma~\ref{lem:kernel-Lip}]
Let $\psi_K(u,s):=K\,u/(u+s)$ and abbreviate $G_{r}(x):=\mathbb{E}[\psi_K(\widetilde U_r,\widetilde S_r)]$, 
where the expectation is over $X_1,\dots,X_{K-1}\stackrel{\text{i.i.d.}}{\sim}p$, any auxiliary randomness (e.g., $\varepsilon,\varepsilon_j$), $\widetilde{U}_r=e^{r(x)+\varepsilon(x)}$ and $\widetilde{S}_r=\sum_{j=1}^{K-1}e^{r(X_j)+\varepsilon_j(X_j)}$. 
Set $r_t(\cdot)=r(\cdot)+t\,\Delta r(\cdot)$, $H_t(x)=G_{r_t}(x)$, $U_t=\widetilde{U}_{r_t}=e^{r_t(x)+\varepsilon(x)}$, and $S_t=\widetilde{S}_{r_t}=\sum_{j=1}^{K-1}e^{r_t(X_j)+\varepsilon_j(X_j)}$.
We have
\[
  \partial_u\psi_K(u,s)=\frac{Ks}{(u+s)^2},\qquad
  \partial_s\psi_K(u,s)=-\frac{Ku}{(u+s)^2}.
\]
By the chain rule and dominated convergence (see bound below),
\[
  \frac{d}{dt}H_t(x)
  =\mathbb{E}\!\left[\partial_u\psi_K(U_t,S_t)\,\frac{dU_t}{dt}
                   +\partial_s\psi_K(U_t,S_t)\,\frac{dS_t}{dt}\right],
\]
with $\frac{d}{dt}U_t=\Delta r(x)\,U_t$ and 
$\frac{d}{dt}S_t=\sum_{j=1}^{K-1}\Delta r(X_j)\,e^{r_t(X_j)+\varepsilon_j(X_j)}$.
Taking absolute values and using $|\Delta r(\cdot)|\le\|\Delta r\|_\infty$,
\[
\begin{aligned}
\Bigl|\tfrac{d}{dt}H_t(x)\Bigr|
&\le K\|\Delta r\|_\infty\,\mathbb{E}\!\left[
      \frac{U_tS_t}{(U_t+S_t)^2}
      +\frac{U_t}{(U_t+S_t)^2}\sum_{j=1}^{K-1} e^{r_t(X_j)+\varepsilon_j(X_j)}
     \right] \\
&= K\|\Delta r\|_\infty\,\mathbb{E}\!\left[\frac{2U_tS_t}{(U_t+S_t)^2}\right]
 \;\le\; \frac{K}{2}\,\|\Delta r\|_\infty,
\end{aligned}
\]
since $0\le 2us/(u+s)^2\le \tfrac12$ for all $u,s>0$. Integrating from $t=0$ to $t=1$ gives
\[
|H_1(x)-H_0(x)|\le \int_0^1 \Bigl|\tfrac{d}{dt}H_t(x)\Bigr|\,dt
\le \frac{K}{2}\,\|\Delta r\|_\infty.
\]
Finally, $H_0(x)=H_{p,0}^{K}(x)$ and $H_1(x)=H_{p,\Delta r}^{K}(x)$, which yields the claim.
\end{proof}

\begin{proof}[Proof of Theorem~\ref{thm:stable_1}]
As a direct consequence of Lemma~\ref{lem:kernel-Lip}, we have
\begin{align*}
    \delta(\eta) &= \sup_{\| \Delta r \|_\infty\leq \eta } \sup_{p\in \mathcal{P}_{\pi}} d_{\rm TV}(T_{\Delta r}p , T_{0}p )\\
    &=(1-\alpha)\sup_{\| \Delta r \|_\infty\leq \eta } \sup_{p\in \mathcal{P}_{\pi}} d_{\rm TV}(pH_{p,\Delta r}^{K} , pH_{p, 0}^{K} )\\
    &=\frac{1-\alpha}{2}\sup_{\| \Delta r \|_\infty\leq \eta } \sup_{p\in \mathcal{P}_{\pi}}  \int_{\mathcal{X}} p(x)|H_{p,\Delta r}^{K}(x)-H_{p,0}^{K}(x)| \pi(dx)\\
    &\leq \frac{\rho \eta}{4}.
\end{align*}
By Lemma~\ref{lem:stability_11_general}, we complete the proof.
\end{proof}

\subsection{Stable Regime (iv)}
We first consider the a general setting for discrete dynamics under perturbations.

\begin{lemma}\label{lem:stability_2_general}
   Assume that
\begin{align}
    \limsup_{\substack{\|\Delta r\|_{\infty}\to 0\\ d_{\rm TV}(w', w)\to 0 }} &d_{\rm TV}(\mathfrak{L}_{\Delta r} w',\mathfrak{L}_{0} w)=0, \text{ for any }w\in \mathcal{P}_{\mathbb{P}_{\rm ref}},\label{lim:45_joint_continuous}\\
    \limsup_{\|\Delta r\|_{\infty}\to 0}\quad &d_{\rm TV}(\mathfrak{L}_{\Delta r}^\infty w_0, \mathfrak{L}_{0}^\infty w_0)=0,\label{lim:46_limit_continuous}
\end{align}
and
\begin{equation}\label{lim:47}
   \limsup_{\eta\to 0^+} \limsup_{N\to \infty}\sup_{\|\Delta r\|_{\infty}\leq \eta } d_{\rm TV}(\mathfrak{L}_{\Delta r}^N w_0, \mathfrak{L}_{\Delta r}^\infty w_0)=0.
\end{equation}
Then 
\begin{equation}\label{lim:48_stability}
    \limsup_{\|\Delta r\|_{\infty}\to 0} \sup_{n\in \mathbb{N}\cup \{\infty\} } d_{\rm TV}(\mathfrak{L}_{\Delta r}^n w_0,\mathfrak{L}_{0}^n w_0) = 0,
\end{equation}
where, by convention,
\[
\limsup_{\|\Delta r\|_{\infty}\to 0} h(\Delta r)
:=\lim_{\eta\to 0^+} \; \sup_{\|\Delta r\|_{\infty}\le \eta} h(\Delta r).
\]
\end{lemma}

\begin{proof}[Proof of Lemma~\ref{lem:stability_2_general}]
Let 
\[
I=\limsup_{\|\Delta r\|_{\infty}\to 0} \sup_{n\in \mathbb{N} } d_{\rm TV}(\mathfrak{L}^n_{\Delta r} w_0,\mathfrak{L}^n_{0} w_0).
\]
It suffices to show that $I=0$. For any $N\in\mathbb{N}$,
\begin{equation}\label{ineq:49}
    I\leq \limsup_{\|\Delta r\|_{\infty}\to 0} \max_{1\leq n\leq N-1 } d_{\rm TV}(\mathfrak{L}^n_{\Delta r} w_0,\mathfrak{L}^n_{0} w_0) +  \limsup_{\|\Delta r\|_{\infty}\to 0} \sup_{n\geq N} d_{\rm TV}(\mathfrak{L}^n_{\Delta r} w_0,\mathfrak{L}^n_{0} w_0).
\end{equation}
Define $F:(\Delta r, w)\mapsto \mathfrak{L}_{\Delta r} w$. Assumption \eqref{lim:45_joint_continuous} is precisely the joint continuity
\[
F(\Delta r, w')=\mathfrak{L}_{\Delta r} w'\to \mathfrak{L}_0 w=F(0, w),\text{ as }\Delta r\to 0, w'\to w,
\]
that is 
\[
\limsup_{\substack{\|\Delta r\|_{\infty}\to 0\\ d_{\rm TV}(w',w)\to 0}}d_{\rm TV} (F(\Delta r, w'),F(0, w))=0.
\]
We claim that for each fixed $n\in \mathbb{N}$ and $w\in \mathcal{P}_{\mathbb{P}_{\rm ref}}$,
\begin{equation}\label{equ:49_induction}
    \lim_{\Delta r \to 0}d_{\rm TV}(\mathfrak{L}^n_{\Delta r}w, \mathfrak{L}_0^nw)=0.
\end{equation}
For $n=1$ this is \eqref{lim:45_joint_continuous}. If $\mathfrak{L}^n_{\Delta r}w\to \mathfrak{L}_0^nw$ as $\Delta r\to 0$, then
\[
\mathfrak{L}^{n+1}_{\Delta r}w=F(\Delta r, \mathfrak{L}^{n}_{\Delta r}w)\to F(0,\mathfrak{L}^{n}_{0}w)=\mathfrak{L}_0^{n+1}w\text{ as $\Delta r\to 0$}.
\]
Thus, the first term on the right-hand side of \eqref{ineq:49} vanishes. Hence, for any $N>0$,
\begin{equation}\label{ineq:51}
    I\leq   \limsup_{\|\Delta r\|_{\infty}\to 0} \sup_{n\geq N} d_{\rm TV}(\mathfrak{L}^n_{\Delta r} w_0,\mathfrak{L}^n_{0} w_0)=\limsup_{\eta\to 0^+} \sup_{\|\Delta r\|_{\infty}\leq \eta} \sup_{n\geq N} d_{\rm TV}(\mathfrak{L}^n_{\Delta r} w_0,\mathfrak{L}^n_{0} w_0).
\end{equation}
By the triangle inequality and subadditivity of the supremum, we have
\begin{equation*}
    \sup_{\|\Delta r\|_{\infty}\leq \eta} \sup_{n\geq N} d_{\rm TV}(\mathfrak{L}^n_{\Delta r} w_0,\mathfrak{L}^n_{0} w_0)\leq 2\sup_{\|\Delta r\|_{\infty}\leq \eta} \sup_{n\geq N} d_{\rm TV}(\mathfrak{L}^n_{\Delta r} w_0,\mathfrak{L}^\infty_{\Delta r} w_0)+\sup_{\|\Delta r\|_{\infty}\leq \eta}  d_{\rm TV}(\mathfrak{L}^\infty_{\Delta r} w_0,\mathfrak{L}^\infty_{0} w_0).
\end{equation*}
Therefore, for any fixed $\eta>0$ and $N\ge 0$,
\begin{equation*}
    I\leq 2\sup_{n\geq N}  \sup_{\|\Delta r\|_{\infty}\leq \eta } d_{\rm TV}(\mathfrak{L}^n_{\Delta r} w_0,\mathfrak{L}^\infty_{\Delta r} w_0)+\limsup_{\|\Delta r\|_{\infty}\to  0 }  d_{\rm TV}(\mathfrak{L}^\infty_{\Delta r} w_0,\mathfrak{L}^\infty_{0} w_0).
\end{equation*}
Letting $N\to \infty$ and $\eta\to 0^+$, we obtain
\begin{equation*}
    I\leq 2\limsup_{\eta\to 0^+}\limsup_{ N\to \infty}  \sup_{\|\Delta r\|_{\infty}\leq \eta } d_{\rm TV}(\mathfrak{L}^N_{\Delta r} w_0,\mathfrak{L}^\infty_{\Delta r} w_0)+\limsup_{\|\Delta r\|_{\infty}\to  0 }  d_{\rm TV}(\mathfrak{L}^\infty_{\Delta r} w_0,\mathfrak{L}^\infty_{0} w_0).
\end{equation*}
As a consequence of \eqref{lim:46_limit_continuous} and \eqref{lim:47}, we conclude that $I=0$, which completes the proof of \eqref{lim:48_stability}.
    
\end{proof}

\begin{lemma}\label{lem:51_proof}
For any $\Delta r\in L^\infty(\mathcal X,\mathcal B,\pi)$ and $w,w'\in \mathcal{P}_{\rm ref}$, we have
\begin{equation*}
    d_{\rm TV}(\mathfrak{L}_{\Delta r} w' ,\mathfrak{L}_{0} w) \leq \frac{(1-\alpha)Q_*}{\innerproduct{w}{Q_0}_{\rm ref}} \Big(\exp(\| {\Delta r}   \|_\infty )-1+   2\cdot d_{\rm TV}(w,w') \Big),
\end{equation*}
and
\begin{equation}\label{equ:Hilbert_projective_norm_continuous}
    d_{\mathcal{H}}(\mathfrak{L}_{\Delta r} w' ,\mathfrak{L}_{0} w) \leq d_{\mathcal{H}}( w',  w) +    2\cdot \|\Delta r \|_\infty  .
\end{equation}
\end{lemma}
\begin{proof}[Proof of Lemma~\ref{lem:51_proof}]
Set $\| f \|_1=\int_{x\in \mathcal{X}} |f(x)|\mathbb{P}_{\rm ref}(dx)$. We have $\innerproduct{w'}{Q_{\Delta r}}_{\rm ref}=\| w'Q_{\Delta r} \|_1$. By triangle inequalities, we have
\begin{align*}
        \| \mathfrak{L}_{\Delta r} w' -\mathfrak{L}_{0} w \|_1 &= (1-\alpha) \left\| \frac{w'Q_{\Delta r}}{\| w'Q_{\Delta r}\|_1   } -  \frac{wQ_{0}}{\| wQ_{0}\|_1   } \right\|_1\\
        &\leq (1-\alpha) \frac{\| wQ_{0} -w'Q_{\Delta r} \|_1 +\Big| \| wQ_{0}\|_1-\| w'Q_{\Delta r}\|_1 \Big| }{\| wQ_{0}\|_1}\\
        &\leq (1-\alpha) \frac{2\| Q_{0} -Q_{\Delta r} \|_\infty+2Q_* \|w'-w\|_1   }{\| wQ_{0}\|_1}\\
        &\leq (1-\alpha)Q_* \frac{2\exp(\| {\Delta r} \|_\infty)-2+2\cdot  \|w'-w\|_1   }{\| wQ_{0}\|_1}
\end{align*}
Thus,
\begin{equation*}
         d_{\rm TV}(\mathfrak{L}_{\Delta r} w'(x),\mathfrak{L}_{0} w(x)) = \frac{1}{2}\| \mathfrak{L}_{\Delta r} w' -\mathfrak{L}_{0} w \|_1\leq (1-\alpha)Q_* \frac{\exp(\| {\Delta r} \|_\infty)-1 +2 d_{\rm TV}(w',w)}{ \|wQ_0 \|_1 }.
\end{equation*}    

\begin{claim}\label{claim:B6}
If $Q(x),Q'(x)\geq 0$ and $w,w'$ are PDFs, then
\[
d_{\mathcal{H}}\Big( \frac{Q'w'}{\innerproduct{w'}{Q'}_{\rm ref}}  , \frac{Qw}{\innerproduct{w}{Q}_{\rm ref}} \Big)\leq d_{\mathcal{H}}(  w' , w )+2 \Big\|\log \frac{Q'}{Q} \Big\|_\infty.
\]
\end{claim}
\begin{proof}[Proof of Claim~\ref{claim:B6}]
    Because the Hilbert projective metric is invariant under positive multiplication, that is,
\begin{equation}
    d_{\mathcal{H}}\Big( \frac{Q'w'}{\innerproduct{w'}{Q'}_{\rm ref}}  , \frac{Qw}{\innerproduct{w}{Q}_{\rm ref}} \Big) = d_{\mathcal{H}} (Q'w',Qw) \leq d_{\mathcal{H}} (w',w)+ 2\Big\|\log \frac{Q'}{Q} \Big\|_\infty.
\end{equation}
\end{proof}

\begin{claim}\label{claim:B7}
Fix $\alpha \in [0,1]$. For any positive density functions $f,g$ with respect to $P_{\rm ref}$ let $T(f)=\alpha \mathbf 1+(1-\alpha)f$. Then
\[
d_{\mathcal{H}}(Tf,Tg)\leq d_{\mathcal{H}}(f,g).
\]
\end{claim}
\begin{proof}[Proof of Claim~\ref{claim:B7}]
    Because $f$ and $g$ are probability densities, we have $\operatorname*{ess\,inf}_{x\in \mathcal{X}} \frac{f(x)}{g(x)} \leq 1 \leq \operatorname*{ess\,sup}_{x\in \mathcal{X}} \frac{f(x)}{g(x)}$. Due to 
\begin{equation*}
    \frac{Tf(x)}{Tg(x)} = \frac{\alpha+(1-\alpha)g(x) (f(x)/g(x))}{\alpha+(1-\alpha)g(x)  },
\end{equation*}
we obtain
\begin{equation*}
    \operatorname*{ess\,inf}_{x\in \mathcal{X}} \frac{f(x)}{g(x)} \leq \operatorname*{ess\,inf}_{x\in \mathcal{X}} \frac{Tf(x)}{Tg(x)}\leq \operatorname*{ess\,sup}_{x\in \mathcal{X}} \frac{Tf(x)}{Tg(x)} \leq \operatorname*{ess\,sup}_{x\in \mathcal{X}} \frac{f(x)}{g(x)}, 
\end{equation*}
which implies $d_{\mathcal{H}}(Tf,Tg)\leq d_{\mathcal{H}}(f,g)$.
\end{proof}

Hence, as a consequence of Claims~\ref{claim:B6} and~\ref{claim:B7}, we obtain
\begin{equation}
     d_{\mathcal{H}}(\mathfrak{L}_{\Delta r} w' ,\mathfrak{L}_{0} w)  \leq d_{\mathcal{H}}\Big( \frac{Q_{\Delta r}w'}{\innerproduct{w'}{Q_{\Delta r}}_{\rm ref}}  , \frac{Q_0w}{\innerproduct{w}{Q_0}_{\rm ref}} \Big)  \leq  d_{\mathcal{H}}(  w' ,  w) +2\| \Delta r\|_\infty.
\end{equation}

\end{proof}

Let $B_\infty(\eta_*)=\{\Delta r\in L^\infty(\mathcal X,\mathcal B,\pi):\|\Delta r\|_\infty\leq \eta_* \}$ denote the ball of radius $\eta_*$ in $L^\infty(\mathcal X,\mathcal B,\pi)$.

For any perturbation \(\Delta r\in \operatorname{int}(B_\infty(\eta_*))\), Lemma~\ref{lem:fixed_point_c} guarantees that \(\mathfrak{L}_{\Delta r}\) admits a unique fixed point, which we denote by \(w_{*,\Delta r}\). Moreover, there exists a constant \(c_{*,\Delta r}>0\) such that
\begin{equation}\label{eq:wstar}
    w_{*,\Delta r}(x)
    = \frac{\alpha}{1 -  (1-\alpha)\,Q_{\Delta r}(x)/c_{*,\Delta r}}.
\end{equation}

\begin{lemma}[Continuity of the implicit root]\label{lem:root-continuous}
Fix a constant $\eta_{*}>0$ and, for each $\Delta r\in \operatorname{int}(B_\infty(\eta_*))$, let real numbers
\(c_{\min,{\Delta r}}\;<\;c_{\max,{\Delta r}}.\)
Assume
\[
u:\Bigl\{(c,{\Delta r}): \|\Delta r \|_\infty  < \eta_{*},
           \;c\in(c_{\min,{\Delta r}},c_{\max,{\Delta r}}]
      \Bigr\}\longrightarrow\mathbb R
\]
satisfies
\begin{enumerate}\itemsep3pt
\item[\textup{(i)}] $(c,{\Delta r})\mapsto u(c,{\Delta r})$ is continuous;
\item[\textup{(ii)}] for every fixed ${\Delta r}\in B_\infty(\eta_*)$,
                     the map $c\mapsto u(c,{\Delta r})$ is strictly decreasing on $(c_{\min,{\Delta r}},c_{\max,{\Delta r}}]$;
\item[\textup{(iii)}] $u\bigl(c_{\max,{\Delta r}},{\Delta r}\bigr)\le 1\ \text{and}\  u\bigl(c_{\min,{\Delta r}}+0,{\Delta r}\bigr)>1$ for all $\Delta r\in B_\infty(\eta_*)$.
\end{enumerate}
Then for every $\Delta r\in \operatorname{int}(B_\infty(\eta_*))$, there exists a unique \(c_{*,\Delta r}\in(c_{\min,\Delta r},c_{\max,\Delta r}]
\ \text{such that } u(c_{*,\Delta r},\Delta r)=1;\) and the map $\Delta r\longmapsto c_{*,\Delta r}$ is continuous on $\operatorname{int}(B_\infty(\eta_*))$.
\end{lemma}

\begin{proof}[Proof of Lemma~\ref{lem:root-continuous}]
\textbf{1.\ Existence and uniqueness for fixed $\Delta r$.}
For any $\Delta r\in B_\infty(\eta_*)$ define
\[
   F_{\Delta r}(c):=u(c,{\Delta r})-1,
   \qquad c\in(c_{\min,{\Delta r}},c_{\max,{\Delta r}}].
\]
By \textup{(ii)} each $F_{\Delta r}$ is strictly decreasing, while \textup{(iii)} gives
$\lim_{c\to c^+_{\min,\Delta r}}F_{\Delta r}(c)>0$ and $F_{\Delta r}(c_{\max,{\Delta r}})\le0$.
By the intermediate value theorem there exists at least one
$c_{*,{\Delta r}}\in(c_{\min,{\Delta r}},c_{\max,{\Delta r}}]$ with $F_{\Delta r}(c_{*,{\Delta r}})=0$,
and strict monotonicity yields uniqueness.

\smallskip
\textbf{2.\ Continuity of $\Delta r\mapsto c_{*,\Delta r}$.}
Fix $\bar r\in \operatorname{int}(B_\infty(\eta_*))$ and write $c_*:=c_{*,\bar r}$.
There exist $\varepsilon_0,\delta_0>0$ such that
\begin{equation}\label{eq:feasible-cushion}
[c_*-\varepsilon_0,\min \{c_*+\varepsilon_0, c_{\max,\Delta r}\}]\subset(c_{\min,\Delta r},c_{\max,\Delta r}]
\quad\text{whenever }\|\Delta r-\bar r\|_\infty<\delta_0.
\end{equation}
%(Condition \eqref{eq:feasible-cushion} holds, for example, if the endpoints are independent of $\Delta r$, or if $u(c_{\max,\bar r},\bar r)<1$ and the feasible-interval correspondence is inner semicontinuous at $\bar r$.)

Let $\varepsilon\in(0,\varepsilon_0)$ and set $c_-:=c_*-\varepsilon$, $c_+:=\min \{c_*+\varepsilon, c_{\max,\Delta r}\}$.
By \eqref{eq:feasible-cushion}, $(c_-,\Delta r)$ and $(c_+,\Delta r)$ belong to the domain of $u$ for all $\Delta r$ sufficiently close to $\bar r$.
Since $u$ is continuous in both arguments by \textup{(i)} and $c\mapsto u(c,\bar r)$ is strictly decreasing by \textup{(ii)},
we have
\[
u(c_-,\bar r)>1\geq u(c_+,\bar r).
\]
Hence, by joint continuity, there exists $\delta\in(0,\delta_0]$ such that for all $\Delta r$ with $\|\Delta r-\bar r\|_\infty<\delta$,
\[
u(c_-,\Delta r)>1
\qquad\text{and}\qquad
u(c_+,\Delta r)\leq 1.
\]
For each such $\Delta r$, the strict decrease of $c\mapsto u(c,\Delta r)$ implies
\[
c_-<c_{*,\Delta r}\leq c_+.
\]
Therefore $|c_{*,\Delta r}-c_*|<\varepsilon$ whenever $\|\Delta r-\bar r\|_\infty<\delta$, which proves continuity at $\bar r$.
Since $\bar r$ was arbitrary, $\Delta r\mapsto c_{*,\Delta r}$ is continuous on $\operatorname{int}(B_\infty(\eta_*))$.
\end{proof}

\begin{lemma}\label{lem:52_proof}
Consider the setup in Theorem~\ref{thm:stable_2}. For any $\Delta r \in \operatorname{int}(B_\infty(\eta_*))$, under the TV metric we have
\begin{equation}\label{equ:limit_B.7}
    \mathfrak{L}^\infty_{\Delta r}w_0=\lim_{n\to\infty}\mathfrak{L}^n_{\Delta r}w_0=w_{*,\Delta r}.
\end{equation}
Moreover, we have
\[
\lim_{\Delta r\to 0} d_{\rm TV}(w_{*,0},w_{*,\Delta r})=0 .
\]
\end{lemma}

\begin{proof}[Proof of Lemma~\ref{lem:52_proof}]
Set
\(
h(x;c,\Delta r)=\frac{\alpha}{\,1-(1-\alpha)Q_{\Delta r}(x)/c\,}.
\)
Note that $Q_{*,\Delta r}$ is Lipschitz in $\Delta r$. Fix $\Delta r$. By Assumption~\ref{ass:4_perturbation} we have
\[
\innerproduct{h(\cdot;Q_{*,\Delta r},\Delta r)}{\mathbf 1}_{\rm ref}\le 1
\quad\text{and}\quad
\innerproduct{h(\cdot;(1-\alpha)Q_{*,\Delta r}+0,\Delta r)}{\mathbf1}_{\rm ref}>1 .
\]
Define $u(c,\Delta r)=\innerproduct{h(\cdot;c,\Delta r)}{\mathbf1}_{\rm ref}$. It is obvious that $u(Q_{*,\Delta r},\Delta r)\le 1$ and $u((1-\alpha)Q_{*,\Delta r}+0,\Delta r)>1$. Moreover, $u$ is continuous over $\{(c,\Delta r): c\in((1-\alpha)Q_{*,\Delta r},Q_{*,\Delta r}],\,\|\Delta r\|_\infty<\eta_*\}$, and for fixed $\Delta r$, $u(c,\Delta r)$ is strictly decreasing in $c$.

By Lemma~\ref{lem:root-continuous}, for each $\Delta r$ there exists a unique
$c_{*,\Delta r}\in\bigl((1-\alpha)Q_{*,\Delta r},\,Q_{*,\Delta r}\bigr]$ such that $u(c_{*,\Delta r},\Delta r)=1$; moreover, the mapping $\Delta r\mapsto c_{*,\Delta r}$ is continuous on $\operatorname{int}(B_\infty(\eta_*))$. One can easily check that $w_{*,\Delta r}=h(x;c_{*,\Delta r},\Delta r)$ is the fixed point of $\mathfrak{L}_{\Delta r}$ (see also the proof of Lemma~\ref{lem:fixed_point_c}).

By Theorem~\ref{thm:Regime_iv_convergence}, \eqref{equ:limit_B.7} holds in the Hilbert projective metric, where $w_k$ and $w_*$ are replaced by $\mathfrak{L}^n_{\Delta r}w_0$ and $w_{*,\Delta r}$, respectively. Since the TV metric is naturally dominated by the Hilbert projective metric (see Lemma 1 in \citet{atar1997exponential}), i.e., for some universal constant $C>0$,
\begin{equation}\label{ineq:Hilbert_projective_metric_TV}
    d_{\rm TV}\bigl(\mathfrak{L}^n_{\Delta r}w_0,\,w_{*,\Delta r}\bigr)
\le C\cdot d_{\mathcal{H}}\bigl(\mathfrak{L}^n_{\Delta r}w_0,\,w_{*,\Delta r}\bigr),
\end{equation}
\eqref{equ:limit_B.7} also holds for the TV metric. Notice that
\begin{align*}
    |w_{*,0}(x)-w_{*,\Delta r}(x)|
    =\left|\frac{(1-\alpha)Q(x)\bigl(1/c_{*,0}-e^{\Delta r(x)}/c_{*,\Delta r}\bigr)}{1-(1-\alpha)Q(x)/c_{*,0}}\right|
      \,w_{*,\Delta r}(x) \le g(\Delta r)\,w_{*,\Delta r}(x),
\end{align*}
where the positive function
\[
g(\Delta r)=\frac{(1-\alpha)Q_*}{1-(1-\alpha)Q_*/c_{*,0}}
\max\left( \big|1/c_{*,0}-e^{\|\Delta r\|_\infty }/c_{*,\Delta r}\big|,\big|1/c_{*,0}-e^{-\|\Delta r\|_\infty }/c_{*,\Delta r}\big|\right).
\]
As a consequence of Lemma~\ref{lem:root-continuous}, we have
\[
\lim_{\|\Delta r\|_\infty\to 0} g(\Delta r)=0 .
\]
Thus,
\[
d_{\rm TV}(w_{*,0},w_{*,\Delta r})
=\frac{1}{2}\int |w_{*,0}(x)-w_{*,\Delta r}(x)|\,\mathbb{P}_{\rm ref}(dx)
\le \frac{g(\Delta r)}{2}\int w_{*,\Delta r}(x)\,\mathbb{P}_{\rm ref}(dx)
=\frac{g(\Delta r)}{2}.
\]
Letting $\Delta r\to 0$, we complete the proof.
\end{proof}

\begin{lemma}\label{lem:53_proof}
Under the setup in Theorem~\ref{thm:stable_2}, we have
\begin{equation}\label{lim:67}
   \limsup_{\eta\to 0^+} \limsup_{N\to \infty}\sup_{\|\Delta r\|_{\infty}\leq \eta  } d_{\rm TV}(\mathfrak{L}^N_{\Delta r} w_0, \mathfrak{L}^\infty_{\Delta r} w_0)=0.
\end{equation}

\end{lemma}

\begin{proof}[Proof of Lemma~\ref{lem:53_proof}]
As a consequence of inequality \eqref{ineq:Hilbert_projective_metric_TV}, to show \eqref{lim:67}, it suffices to show
\begin{equation}\label{lim:65_hp_norm}
   \limsup_{\eta\to 0^+} \limsup_{N\to \infty}\sup_{\|\Delta r\|_{\infty}\leq \eta  } d_{\mathcal{H}}(\mathfrak{L}^N_{\Delta r} w_0, \mathfrak{L}^\infty_{\Delta r} w_0)=0.
\end{equation}
If \eqref{lim:65_hp_norm} fails, then there exists a positive number $\varepsilon_0$, a sequence of positive numbers $\{\eta_n\}_{n=1}^\infty$ converge to zero, a sequence of increasing integers $\{N_n\}_{n=1}^\infty$, and a sequence of functions $\{\Delta r_n\}_{n=1}^\infty$ such that $\|\Delta r_n\|_\infty \leq \eta_n$ for any $n\geq 1$ satisfy 
\[
 d_{\mathcal{H}}(\mathfrak{L}^{N_n}_{\Delta r_n} w_0, \mathfrak{L}^\infty_{\Delta r_n} w_0)\geq \varepsilon_0>0\text{ and } \lim_{n\to \infty}\| \Delta r_n\|_\infty = 0.
\]
Set $F:(\Delta r, w)\mapsto \mathfrak{L}_{\Delta r} w$. The inequality \eqref{equ:Hilbert_projective_norm_continuous} implies the joint continuity under the Hilbert projective metric
\[
F(\Delta r, w')=\mathfrak{L}_{\Delta r} w'\to L_0 w=F(0, w),\text{ as }\|\Delta r\|_\infty\to 0, d_{\mathcal{H}}(w', w)\to 0.
\]
Similar to the proof of \eqref{equ:49_induction}, we can show that for any integer $l\geq 1$
\begin{equation}\label{equ:limit_69_new}
    \lim_{\Delta r \to 0}d_{\mathcal{H}}(\mathfrak{L}^{l}_{\Delta r} w_0, \mathfrak{L}^{l}_{0} w_0 )=0.
\end{equation}
By Lemma~\ref{lem:52_proof}, we have that $\mathfrak{L}^\infty_{\Delta r}w_0= w_{*,\Delta r}$. If $\| \gamma_{\Delta r} - \gamma_{0} \|_\infty<\gamma_{\min}$, then
\begin{align*}
    d_{\mathcal{H}}(w_{*,\Delta r},w_{*,0})&=\operatorname*{ess\,sup}_{x\in \mathcal{X}}\log \frac{ \gamma_{\Delta r}(x) }{\gamma_{0}(x)  } - \operatorname*{ess\,inf}_{x\in \mathcal{X}}\log \frac{ \gamma_{\Delta r}(x) }{\gamma_{0}(x)  }\\
    &\leq \operatorname*{ess\,sup}_{x\in \mathcal{X}} \log\frac{\gamma_{0}(x)+\| \gamma_{\Delta r} - \gamma_{0} \|_\infty }{\gamma_{0}(x)  } - \operatorname*{ess\,inf}_{x\in \mathcal{X}} \log\frac{\gamma_{0}(x)-\| \gamma_{\Delta r} - \gamma_{0} \|_\infty }{\gamma_{0}(x)  }\\
    &\leq  \log\frac{\gamma_{\min}+\| \gamma_{\Delta r} - \gamma_{0} \|_\infty }{\gamma_{\min}  } - \log\frac{\gamma_{\min}-\| \gamma_{\Delta r} - \gamma_{0} \|_\infty }{\gamma_{\min}  }
\end{align*}
where $\gamma_{\Delta r}(x)=1-(1-\alpha)Q_{\Delta r}(x)/c_{*,\Delta r}$ and $\min_{x\in \mathcal{X}}\gamma_0(x)=1-(1-\alpha)Q_*/c_{*,0}=:\gamma_{\min}>0$. In Lemma~\ref{lem:root-continuous}, we know that $c_{*,\Delta r}$ is continuous in $\Delta r$. Hence
\begin{align*}
        \| \gamma_{\Delta r} - \gamma_{0} \|_\infty\leq Q_* \Big\| \frac{e^{\Delta r(x)}}{c_{*,\Delta r}} -\frac{1}{c_{*,0}} \Big\|_{\infty}\to 0,
\end{align*}
as $\|\Delta r \|_\infty\to 0$. In conclusion,
\begin{equation}\label{equ:limit_69}
    \lim_{\Delta r \to 0}d_{\mathcal{H}}(\mathfrak{L}^{\infty}_{\Delta r} w_0, \mathfrak{L}^{\infty}_{0} w_0 )=\lim_{\Delta r\to 0}d_{\mathcal{H}}(w_{*,\Delta r},w_{*,0}) =0.
\end{equation}

Due to $d_{\mathcal{H}}(\mathfrak{L}^{k+1}_{\Delta r} w_0, \mathfrak{L}^\infty_{\Delta r} w_0) \leq d_{\mathcal{H}}(\mathfrak{L}^{k}_{\Delta r} w_0, \mathfrak{L}^\infty_{\Delta r} w_0)$ for all $k\geq 1$ (see \eqref{equ:45_Birkhoff}), for any fixed integer $l\geq 1$, we have
\begin{equation}\label{ineq:71_final}
    \varepsilon_0\leq\limsup_{n\to \infty} d_{\mathcal{H}}(\mathfrak{L}^{N_n}_{\Delta r_n} w_0, \mathfrak{L}^\infty_{\Delta r_n} w_0) \leq \limsup_{n\to \infty} d_{\mathcal{H}}(\mathfrak{L}^{l}_{\Delta r_n} w_0, \mathfrak{L}^\infty_{\Delta r_n} w_0).
\end{equation}
Due to \eqref{equ:limit_69_new} and \eqref{equ:limit_69}, we have
\[
|d_{\mathcal{H}}(\mathfrak{L}^{l}_{\Delta r_n} w_0, \mathfrak{L}^\infty_{\Delta r_n} w_0) - d_{\mathcal{H}}(\mathfrak{L}^{l}_{0} w_0, \mathfrak{L}^\infty_{0} w_0)|\leq  d_{\mathcal{H}}(\mathfrak{L}^{l}_{\Delta r_n} w_0, \mathfrak{L}^l_{0} w_0) + d_{\mathcal{H}}(\mathfrak{L}^{\infty}_{\Delta r} w_0, \mathfrak{L}^\infty_{0} w_0)\to 0,
\]
as $n\to \infty$. Hence, \eqref{ineq:71_final} implies that for any $l\geq 1$,
\[
d_{\mathcal{H}}(\mathfrak{L}^{l}_{0} w_0, \mathfrak{L}^\infty_{0} w_0) \geq \varepsilon_0>0,
\]
which contradicts $\lim_{l\to \infty}d_{\mathcal{H}}(\mathfrak{L}^{l}_{0} w_0, \mathfrak{L}^\infty_{0} w_0)=0$. 

This completes the proof of \eqref{lim:65_hp_norm}, hence for~\eqref{lim:67}. 
\end{proof}

\begin{proof}[Proof of Theorem~\ref{thm:stable_2}]

By Lemma~\ref{lem:stability_2_general}, it suffices to verify that \eqref{lim:45_joint_continuous}, \eqref{lim:46_limit_continuous}, and \eqref{lim:47} hold for the perturbed update map \eqref{equ:dynamic_L_regime_iv}. As the direct consequence of Lemmas~\ref{lem:51_proof}, \ref{lem:52_proof} and \ref{lem:53_proof}, we complete the proof of Theorem~\ref{thm:stable_2}.   
\end{proof}

\end{document}